\documentclass[12pt,onecolumn,draftclsnofoot,letterpaper]{IEEEtran}
\IEEEoverridecommandlockouts
\usepackage{cite}
\usepackage{bm,bbm,amsthm,amsmath,amssymb,amsfonts,mathrsfs}
\usepackage{algorithm}
\usepackage{algpseudocode}
\usepackage{graphicx}
\usepackage{subfigure}  
\usepackage{textcomp}
\usepackage{xcolor}
\usepackage{booktabs}
\usepackage{boldline}
\usepackage{url}
\usepackage{color}
\usepackage{eqparbox}
\usepackage{hyperref}
\allowdisplaybreaks[4]

\newcommand{\defeq}{\stackrel{\text{\tiny \text{def}}}{=}}

\newtheorem{theorem}{Theorem}
\newtheorem{proposition}{Proposition}
\newtheorem{lemma}{Lemma}

\begin{document}
	\title{Fundamental Limits of Reinforcement Learning in Environment with Endogeneous and Exogeneous Uncertainty}
	\author{\IEEEauthorblockN{Rongpeng Li}
	\thanks{R. Li is with College of Information Science \& Electronic Engineering, Zhejiang University, Hangzhou 310027, China (email: lirongpeng@zju.edu.cn). Part of this work is done when R. Li was a visiting scholar at College of Computer Science and Technology, The University of Cambridge, UK.}}	
	\maketitle
	\begin{abstract}
		Online reinforcement learning (RL) has been widely applied in information processing scenarios, which usually exhibit much uncertainty due to the intrinsic randomness of channels and service demands. In this paper, we consider an un-discounted RL in general Markov decision processes (MDPs) with both endogeneous and exogeneous uncertainty, where both the rewards and state transition probability are unknown to the RL agent and evolve with the time as long as their respective variations do not exceed certain dynamic budget (i.e., upper bound). We first develop a variation-aware Bernstein-based upper confidence reinforcement learning (VB-UCRL), which we allow to restart according to a schedule dependent on the variations. We successfully overcome the challenges due to the exogeneous uncertainty and establish a regret bound of saving at most $\sqrt{S}$ or $S^{\frac{1}{6}}T^{\frac{1}{12}}$ compared with the latest results in the literature, where $S$ denotes the state size of the MDP and $T$ indicates the iteration index of learning steps.
	\end{abstract}
	
	
	\section{Introduction}
	Reinforcement learning (RL) \cite{sutton_reinforcement_1998}, one category of machine learning, has manifested itself in empowering an agent to interact with an environment with much uncertainty. Besides its success in AlphaGo \cite{silver_mastering_2016}, RL has been widely applied in solving information processing problems in areas including wireless communications \& networking \cite{xu_experience-driven_2019,he_trust-based_2020,somuyiwa_reinforcement-learning_2018,nasir_multi-agent_2019,he_software-defined_2017,yu_deep-reinforcement_2019,istepanian_medical_2009,xu_reinforcement_2020,li_smartcc_2019,he_joint_2019,chen_multi-tenant_2019,tang_deep_2020,luong_applications_2019,hua_gan-powered_2020,li_intelligent_2017}, computer vision and robot navigation \cite{bernstein_reinforcement_2018,bernstein_reinforcement_2018-1}. Typically, the RL formalizes this interaction between the agent and the environment in these scenarios through an ``economic" perspective \cite{fruit_exploration-exploitation_2019}. In other words, RL often models the problem as a Markov decision process (MDP), consisting of a tuple of state, action, reward (or equivalently loss) and transition probability, and an RL agent tries to maximize the cumulative rewards or minimizing the cumulative loss, by observing the environment as a state and taking an action accordingly. For example, \cite{xu_experience-driven_2019} and \cite{li_smartcc_2019} try to achieve larger throughput, by considering the networking factors (e.g., goodput, average round-trip time (RTT), data rate) as the state and learning a RL policy to adjust the congestion window of transmission control protocol (TCP). Meanwhile, \cite{he_joint_2019} aims to maximize the rate by regarding the channel information as the state and leveraging RL to make channel assignment and power allocation actions. Note-worthily, in these scenarios, online RL emerges as a popular option. Therefore, it naturally raises a question what is the fundamental performance limit of online RL-based solutions regardless of the specific RL applications?

	The difficulty to know this limit mainly lies in the \emph{endogeneous} and \emph{exogeneous} uncertainty in the MDP. Specifically, in the classical time-homogeneous MDP settings, only endogeneous uncertainty is considered. In other words, at each time-step, the reward follows a reward distribution and the subsequent state follows a state transition distribution. Both distributions solely depend on the current state and action and remain fixed along with the temporal variations. Unfortunately, the aforementioned scenarios often face time-varying reward and transition probability distributions, due to non-stationary channels \cite{hua_gan-powered_2020} and service demands. Therefore, the exogeneous uncertainty has to be taken into account. Typically, in order to unveil the uncertainty in the MDP, the RL agent has to explore the MDP to accumulate the related knowledge of those poorly-visited states and actions. As any decision of RL affects the subsequent observations, more exploration usually produces long-term impact yet affects short-term exploitation efficiency, which is also termed as the \emph{exploration-exploitation dilemma} \cite{sutton_reinforcement_1998} originally discussed in the literature of multi-arm bandit (MAB) \cite{lai_asymptotically_1985}. 
	
	There has been intense research interest towards understanding the performance limit of online RL-based solutions for an time-homogeneous MDP. For example, \cite{kearns_near-optimal_2002} talked about the performance of a learned policy, while Jaksch \textit{et al.} gave the performance limit of an RL algorithm during the learning \cite{jaksch_near-optimal_2010}, which is more meaningful for online RL in information processing scenarios. Specifically, Jaksch \textit{et al.} proposed a UCRL2 algorithm (\textit{upper confidence bound for reinforcement learning}) for un-discounted reinforcement learning in communicating MDPs. In other words, UCRL2 implements the paradigm of ``optimism in the face of uncertainty'' and construct plausible MDPs in confidence interval based on the Hoeffding inequality \cite{bartlett_regal_2009} and proves that the total regret of an RL algorithm with respect to an optimal policy could be bounded by $\tilde{O}(DS\sqrt{AT})$, where $\tilde{\mathcal{O}}(\cdot)$ hides the logarithmic factors, $S$ and $A$ denote the size of the state space and action space of the MDP, respectively. $D$ is the diameter of the communicating MDP, indicating the minimal expected number of steps from each state to another state in MDP. Besides, $T$ denotes the iteration of learning steps. Based on UCRL2, many variants have been proposed to generate tighter bounds. \cite{osband_more_2013} proved a more efficient posterior sampling for episodic reinforcement learning and established an $\tilde{O}(\iota S\sqrt{AT})$ bound on the expected regret with the episode length $\iota$. \cite{hao_bootstrapping_2019} proposed a non-parametric and data-dependent algorithm based on the multiplier bootstrap for MAB. 
	\cite{ortner_online_2012} proposed a UCCRL algorithm to derive sublinear regret bounds for finite-horizon un-discounted reinforcement learning in continuous state space. Later, \cite{qian_exploration_2019} focused on an infinite-horizon un-discounted setting and used an exploration bonus to achieve the same regret bound as UCCRL \cite{ortner_online_2012}. 
	
	Until recently, there emerges few light shed on MDP with both endogeneous and exogeneous uncertainty. \cite{li_online_2019} and \cite{li_online_2019-1} talked about online learning for MDP in this non-stationary environment and provided the dynamic regret analysis for exogeneous uncertainty only. \cite{gajane_variational_2019} extended UCRL2 to a variation-aware algorithm and provided performance guarantees for the regret evaluated against the optimal non-stationary policy. \cite{cheung_learning_2019} and \cite{cheung_reinforcement_2020} gave a more comprehensive study of the dynamic regret, and derived a bound of $\tilde{O} \left((V_r^T + V_p^T)^{1/4} S^{2/3} A^{1/2} T^{3/4}\right)$, where $V_r$ and $V_p$ are the dynamic budget (i.e., upper bound) of variations in reward and transition probability functions. 
	
	Compared with the aforementioned research, the contribution of this paper can be summarized as follows.
	\begin{itemize}
		\item We focus on the RL for MDP with both endogeneous and exogeneous uncertainty, which has significant applications in information processing scenarios.
		\item We propose a variation-aware Bernstein-based upper confidence reinforcement learning (VB-UCRL) algorithm, which restarts according to a schedule dependent on the variations in
		the MDP and leverages the empirical Bernstein inequality \cite{audibert_explorationexploitation_2009} to give a tighter bound.
		\item We prove that the VB-UCRL gives a regret bound of $\tilde{O} \left( (V_r + V_p)^{1/3} T^{2/3}  \sqrt{\Gamma S A} \right)$, where $\Gamma$ denotes the maximal number of reachable states for any state-action pair in the MDP. As discussed in Section \ref{sec:discussion}, this bound is tighter than the latest results in \cite{gajane_variational_2019,cheung_reinforcement_2020}.
	\end{itemize}

	The remainder of the paper is organized as follows. In Section \ref{sec:preliminaries}, we introduce some fundamentals of MDPs and useful concentration inequalities. In Section \ref{sec:proSolu} and Section \ref{sec:result}, we formulate the regret problem of RL for MDP with both endogeneous and exogeneous uncertainty, and prove the related results for the proposed VB-UCRL. Section \ref{sec:discussion} discusses the aforementioned results, by comparing with the state-of-the-art results in the literature. We conclude the paper in Section \ref{sec:conclusion}.

	\section{Preliminaries}
	\label{sec:preliminaries}
	\subsection{Fundamentals of MDPs}

	In a time-homogeneous MDP $M = <\mathcal{S},\mathcal{A},r,p,s_1>$ with state space $\mathcal{S}$, action space $\mathcal{A}$ and the initial state $s_1$. Every state-action-pair is characterized by a reward distribution with mean $r(s,a) \in [0,r_{\max}]$ over next states. For simplicity of representation, we denote the size of state space and action space as $S = \vert \mathcal{S} \vert$ and $A = \vert \mathcal{A} \vert$, respectively. Furthermore, we assume the number of reachable states for a state-action pair $(s,a)$ as $\Gamma(s,a) = \Vert p(\cdot|s,a) >0 \Vert_0$ and $\Gamma = \max_{s,a}\Gamma(s,a)$. Notably, in the time-homogeneous MDP, the mean rewards and transition probabilities only depend on the current state and the chosen action. An MDP is called \textit{communicating}, if for any two states $s$, $s^{\prime}$, when starting in $s$ it is possible to reach $s^{\prime}$ with positive probability choosing appropriate actions.
	
	We primarily focus on the infinite-horizon un-discounted MDP settings and try to learn a policy $\pi$ that maximizes 
	\begin{equation}
		\label{eq:gain_initial_def}
		\sup_{\pi \in \Pi} \left\{\liminf_{T\rightarrow+\infty}\mathbb{E}_{s}^{\pi}\left[\frac{1}{T}\sum_{t=1}^{T} r(s_t,a_t) \bigg|s_1 \sim \mu_1 \right] \right\} 
	\end{equation}
	where $\mu_1$ is the state probability of the starting state $s_1$. The set of stationary randomized (resp. deterministic) policies is denoted by $\Pi^{\text{SR}}$ (resp. $\Pi^{\text{SD}})$\footnote{A Markov randomized decision rule $d: \mathcal{S} \rightarrow \mathcal{P}(\mathcal{A})$ maps states to distributions over actions while a Markov deterministic decision rule $d: \mathcal{S} \rightarrow \mathcal{A}$ maps states to actions.}.
	The subset of Markov randomized decision rules is denoted $D^{\text{MR}}$, while the subset of Markov deterministic decision rules is denoted $D^{\text{MD}}$. For any Markov decision rule $d \in D^{\text{MR}}$, $P_d \in \mathbb{R}^{S\times S}$ and $r_d \in \mathbb{R}^S$ denote the transition matrix and reward vector associated with $d$ i.e.,
	$P_d(s'|s) \defeq \sum_{a\in\mathcal{A}_s} d(a|s) p(s'|s,a)$ and $r_d(s) \defeq \sum_{a\in\mathcal{A}_s} d(a|s) r(s,a)$, for all $s, s \in \mathcal{S}$, where $d(a|s)$ is the probability to sample $a$ in state $s$ when using $d$.

	For a stationary policy, i.e., $\pi \in \Pi^{\text{SR}}$, the lim inf in \eqref{eq:gain_initial_def} actually matches the lim sup and the limit is well defined (Section 8.2.1, \cite{puterman_markov_1994}). In other words, any policy $\pi \in \Pi^{\text{SR}}$ has an associated \textit{long-term average reward} (or gain) $g^{\pi}(s)$ and a \textit{bias function} $h(s)$, which is defined as 
	\begin{align}
		g^{\pi}(s) \defeq \lim_{T\rightarrow+\infty}\mathbb{E}_{s}^{\pi}\big[\frac{1}{T}\sum_{t=1}^{T} r(s_t,a_t)\big]
	\end{align}
	and
	\begin{align}
		h^{\pi}(s) & \defeq \underset{T \rightarrow+\infty}{C\textendash\lim } \mathbb{E}_{s}^{\pi}\left[\sum_{t=1}^{T}\left(r\left(s_{t}, a_{t}\right)-g^{\pi}\left(s_{t}\right)\right)\right] \\
		& = \lim_{T\rightarrow+\infty} \frac{1}{T}\sum_{k=1}^{T} \mathbb{E}_{s}^{\pi}\left[\sum_{t=1}^{k}\left(r\left(s_{t}, a_{t}\right)-g^{\pi}\left(s_{t}\right)\right)\right],
	\end{align}
	respectively. Here, ${E}_{s}^{\pi}$ takes an expectation over trajectories generated starting from $s_1 = s$ with action $a_t \sim \pi(s_t)$. The bias $h^{\pi}(s)$ measures the expected total diﬀerence between the reward and the stationary reward in \textit{Cesaro-limit} (denoted by $C\textendash\lim $)\footnote{The Cesaro-limit is always well-deﬁned unlike the “classical” limit as the series may cycle i.e., have several accumulation points or “cluster points”. Note that for policies with an aperiodic chain, the standard limit exists. Another useful theorem for this limit is the theorem of Cesaro Means, that is, if the limit of a sequence $a_1,\cdots,a_n$ exists $\lim_{n\rightarrow\infty} a_n =a$, let $b_n = n^{-1} \sum_i^n a_i$, then $\lim_{n\rightarrow\infty} b_n =a$.}. 

	Accordingly, the difference of bias $h^{\pi}(s') - h^{\pi}(s)$ quantifies the (dis-)advantage of starting from the state $s'$ rather than $s$ under policy $\pi$. Denote $sp(h^{\pi}) \defeq \max_s h^{\pi}(s) - \min_s  h^{\pi}(s)$. We have the following useful propositions.
	\begin{proposition}[Theorem 8.2.6 of \cite{puterman_markov_1994}]
		For any policy $\pi = d^{\infty} \in \Pi^{\text{SR}}$, the gain $g^{\pi}$ and bias $h^{\pi}$ satisfy the following system of \emph{Bellman evaluation equations}:
		\begin{equation}
			g = P_d g \ \text{and}\ h + g = L_d h
			\label{eq:bellman_equation}
		\end{equation}
		where $L_d h \defeq r + P_d h$. Conversely, if $(g,h) \in \mathbb{R}\times\mathbb{R}^S$ is a solution to \eqref{eq:bellman_equation}, then $g = g^{\pi}$ and $h = h^{\pi} + u$ where $u = P_d u$. Finally, if $P_d^{}h = 0$, then $h = h^{\pi}$.
	\end{proposition}
	\begin{proposition}[Chapter 9 of \cite{puterman_markov_1994} and Theorem 1 of \cite{schweitzer_undiscounted_1985}]
		\label{prop:weak_mdp_conv}
		Let $M$ be a weakly communicating MDP and denote $\Pi^{*} \in \Pi^{\text{SD}}$ be the set of policies maximizing \eqref{eq:gain_initial_def}. If any of the following assumptions hold:

		\begin{itemize}
			\item the action space $A$ is \emph{finite},
			\item $\Pi^{*} \neq \emptyset$, and $\sup_{\pi \in \Pi^{*}} < +\infty$,
		\end{itemize}	
		then there exists a solution $(g^{*},h^{*}) \in \mathbb{R}\times\mathbb{R}^S$ to be the fixed point equation $h^{*} + g^{*}e = Lh^{*}$, where $e$ denotes an all-one vector. Moreover, for any such solution  $(g^{*},h^{*})$ and for all $s \in \mathcal{S}$,
		\begin{equation}
			g^{*} = \max_{\pi \in \Pi} \left\{\liminf_{T\rightarrow+\infty}\mathbb{E}_{s}^{\pi}\left[\frac{1}{T}\sum_{t=1}^{T} r(s_t,a_t) \bigg| s_1 \sim \mu_1 \right] \right\} 
		\end{equation} 
		Finally, any stationary policy $\pi^{*} = (d^{*})^{\infty}$ satisfying $d^{*} \in \arg \max_{d} \{r_d + P_d h^{*}\}$ (i.e., greedy policy) is optimal, i.e., $\pi^{*} \in \Pi^{*}$.
	\end{proposition}
	Notably, \cite{puterman_markov_1994} shows that the optimal average reward $g^{*}$ in communicating MDPs is independent of the initial state $s_1$ and cannot be increased when using non-stationary policies.
	
	\begin{proposition}[Theorem 9.4.5 of \cite{puterman_markov_1994} and extension by Theorem 7 of \cite{jaksch_near-optimal_2010}]
		\label{prop:value_iter_converge}
		Consider the sequences of vectors $(v_n)_{n\in \mathbb{N}}$ and Markov decision rules $(d_n)_{n\in \mathbb{N}}$ obtained while executing the value interactions in Alg. \ref{al:relative_value-iteration}. If Prop. \ref{prop:weak_mdp_conv} holds and either of the following conditions:
		\begin{itemize}
			\item every average optimal stationary deterministic policy has an aperiodic transition matrix,
			\item the transition matrices $P_{d_n}$ are aperiodic for all $n \geq 1$
		\end{itemize}
		for all $n \geq 1$, then there exists $h^{*} \in \mathbb{R}^S$ such that the limit of the value function 
		% $v_n \defeq \max_{\pi} \mathbb{E}^{\pi} \left[  \sum_{t=n}^{\infty} r_t \vert s_n =s \right]$ satisfies 
		$\lim_{n\rightarrow\infty} v_n = h^{*}$ (where $v$ is defined using operators $L:\mathbb{R}^S\rightarrow\mathbb{R}^S$ as $Lv\defeq \max_{d \in D^{\text{MR}}} \{r_d+P_d v\}$) and $Lh^{*} = h^{*} + g^{*} e$, where $e$ denotes a all-one vector. 
	\end{proposition}
	\begin{algorithm*}
		\caption{(Relative) Value Iteration}
		\label{al:relative_value-iteration}
		\hspace*{\algorithmicindent} \textbf{Input}: Operators $L:\mathbb{R}^S\rightarrow\mathbb{R}^S$ as $Lv\defeq \max_{d \in D^{\text{MR}}} \{r_d+P_d v\}$ and $G:\mathbb{R}^S\rightarrow D$ as $G v \in \arg \max_{d \in D^{\text{MR}}} \{L_d v\}$, accuracy $\epsilon \in (0, r _{\max})$, initial vector $v_0 \in \mathbb{R}^S$, arbitrary reference state $s \in \mathcal{S}$,
		\begin{algorithmic}[1]
			\State Initialize $n=0$, $v_1 \defeq Lv_0$ 
			\While{$sp(v_{n+1} - v_{n}) > \epsilon$} \algorithmiccomment{Loop until termination}
				\State Increment $n \leftarrow n+1$
				\State Update $v_n \leftarrow v_n - v_n(\bar{s})e $ \algorithmiccomment{Avoids numerical instability ($v_n \nrightarrow \infty$)}
				\State $(v_{n+1},d_n) \defeq (Lv_n, Gv_n)$ \algorithmiccomment{$Lv_n$ and $Gv_n$ can be computed simultaneously}
			\EndWhile
			\State Set $g \defeq \frac{1}{2} \max\{v_{n+1} - v_n\} + \min\{v_{n+1} - v_n\})$, $h \defeq v_n$ and $\pi \defeq (d_n)^{\infty}$.		
		\end{algorithmic}
		\hspace*{\algorithmicindent} \textbf{Output}: Gain $g \in [0, r_{\max}]$, bias vector $h \in \mathbb{R}^S$ and stationary deterministic policy $\pi \in \Pi^{\text{SD}}$.
	\end{algorithm*}
	Also, the operator $L$ is denoted as the optimal Bellman operator. Finally, $D = \max_{s\neq s'}\{ \tau (s \rightarrow s')\}$ denotes the diameter of $M$, where $\tau (s \rightarrow s') = \inf\{t\geq 1:s_t =s^{\prime}|s\}$ is the minimal expected steps required from $s$ to $s'$. Furthermore, for all $ s \in \mathcal{S}$ and the related value function $v \in \mathbb{R}^{S}$,  we define the extended optimal Bellman operator with aperiodic transformation (Proposition 8.5.8, \cite{puterman_markov_1994}) as
	\begin{align}
		\label{eq:extended_opt_bellman_operator_definition}
		L^{\alpha}_{k} v(s)&\defeq \max _{a \in \mathcal{A}_{s}}\left\{\max _{r \in B_{r}^{k}(s, a)}\{r\}+\alpha \cdot \max _{p \in B_{p}^{k}(s, a)}\left\{p^{T} v\right\}\right\}  +(1-\alpha) \cdot v(s)
	\end{align}
	where $\alpha$ is the coefficient of the aperiodic transformation. By properties of the aperiodic transformation, the optimal gains of $\mathcal{M}_k^{\alpha}$ and $\mathcal{M}_k$ are equal (denoted by $g_k^{*}$). The aperiodic transformation makes the extended value iteration in Alg. \ref{al:relative_value-iteration} meet the condition in Prop. \ref{prop:value_iter_converge}. Also, as shown by Prop. 8.5.8 of \cite{puterman_markov_1994}, this transformation does not affect the gain of any stationary policy. In other words, for any $\pi \in \Pi^{\text{SR}}$, $g^{\alpha,\pi} = g^{\pi}$\footnote{The transformation introduced by (Section 8.5.4, \cite{puterman_markov_1994}) is slightly diﬀerent as the rewards are all multiplied by $\alpha$. Therefore, Proposition 8.5.8 of \cite{puterman_markov_1994} states that the gain is also multiplied by $\alpha$. i.e., $g^{\alpha,\pi} = \alpha \cdot g^{\pi}$. However, it is straightforward to adapt the proof of Proposition 8.5.8 of \cite{puterman_markov_1994} to our case.}.

	\begin{proposition}[Theorem 4 of \cite{bartlett_regal_2009}]
		\label{prop:span_h_k}
		Let $M$ be a communicating MDP with non-negative rewards and $(g^{*},h^{*})$ a solution of the Bellman optimality equation (i.e., $Lh^{*} = h^{*}+g^{*}e$). For any states $s$ and $s^{\prime}$ and any stationary policy $\pi \in \Pi^{\text{SR}}$, we have
		\begin{align}
			h^{*}(s^{\prime}) - h^{*}(s) \leq g^{*} \mathbb{E}^{\pi} [\tau(s^{\prime} -1)|s].
		\end{align}
	\end{proposition}
	
	Proposed in Chapter 2, \cite{bertsekas_dynamic_2012}, the stochastic shortest path helps to understand the difficulty for an agent to navigate between the states of an MDP. Let us consider an MDP $M^{\prime} \defeq \{\mathcal{S},\mathcal{A},p,r^{\prime}\}$ with identical state and action space, and transition probabilities as the true MDP $M$ but a reward $r^{\prime} \defeq -r $ for all state-action pairs, which could be further interpreted as the required time or distance before reaching the next state in the MDP. The goal of the stochastic shortest path problem is to find the shortest expected distance between states $s$ and $s^{\prime}$ in the MDP, that is, 
	\begin{align}
		\label{eq:shortest_path_opt}
		\sup\limits_{\pi\in\Pi} \bigg\{\mathbb{E}^{\pi} \big[ \sum_{t=1}^{\tau(s^{\prime})-1} r^{\prime} \big\vert s \big]\bigg\}
	\end{align}
	Though \eqref{eq:shortest_path_opt} appears quite different from the original problem in \eqref{eq:gain_initial_def}, these two problems are related through the Bellman optimality equation with the shortest path problem is regarded as finding a bias-optimal policy with the optimal gain $g^{*} = 0$, thus re-writing the optimality equation as $Lh^{*} = h^{*}$. We have the following proposition.
	\begin{proposition}[Proposition 2.8 of \cite{fruit_exploration-exploitation_2019}]
		\label{prop:bellman_stoc_shortest}
		Let $M^{\prime} \defeq \{\mathcal{S},\mathcal{A},p,r^{\prime}\}$ be communicating MDP (finite or compact $\mathcal{A}$) with $r^{\prime} \in [-r_{\max},0]$ for all $(s,a) \in \mathcal{S}\times\mathcal{A}$. For any state $s \in \mathcal{S}$, consider the \emph{Bellman shortest path operator} $L_{\mapsto s}: \mathbb{R}^S \mapsto \mathbb{R}^S$ defined for all $v \in \mathbb{R}^S$ as $\forall x \in \mathcal{S}$, we have \eqref{eq:bellman_stoc_shortest}, and define $h_{\mapsto s^{\prime}}^{*}$ as the (component-wise) maximal non-positive solution of the Bellman shortest path optimality equation $L_{\mapsto s^{\prime}} h_{\mapsto s^{\prime}}^{*} = h_{\mapsto s^{\prime}}^{*}$. Moreover, if $d_{\mapsto s^{\prime}}^{*}$ is a greedy decision rule with respect to $h_{\mapsto s^{\prime}}^{*}$, namely, $d_{\mapsto s^{\prime}}^{*} = \arg\max_{a \in \mathcal{A}_{s}}\left\{r^{\prime}(s, a)+\sum_{y \in S} p(y | s, a) h_{\mapsto s^{\prime}}^{*}(y)\right\}$ for all $s \neq s^{\prime}$, then $\pi_{\mapsto s^{\prime}}^{*} \defeq (d_{\mapsto s^{\prime}}^{*})^{\infty}$ is an optimal solution to \eqref{eq:shortest_path_opt}.

		\begin{figure*}[!t]
			\begin{align}
				\label{eq:bellman_stoc_shortest}
				L_{\mapsto s^{\prime}} v(s)\defeq \left\{\begin{array}{ll}
					\max _{a \in \mathcal{A}_{s}}\left\{r(s, a)+\sum_{y \in S} p(y | s, a) v(y)\right\} &\text { if } s \neq s^{\prime} \\
					v(s) & \text{ otherwise }
					\end{array}\right.
			\end{align}
			\hrulefill
		\end{figure*}
	\end{proposition}
	\begin{proposition}[Theorem 7.3.2. of \cite{puterman_markov_1994}]
		\label{prop:stoc_shortest_dominance}
		Let $M^{\prime} \defeq \{\mathcal{S},\mathcal{A},p,r^{\prime}$ be communicating MDP (finite or compact $\mathcal{A}$) with negative rewards $r^{\prime}$. For any state $s \in \mathcal{S}$, consider the Bellman shortest path operator $L_{\mapsto s}$ in \eqref{eq:bellman_stoc_shortest} with maximal non-positive fixed point $h_{\mapsto s}^{*}$ in Prop. \ref{prop:bellman_stoc_shortest}. If there exists $h \in \mathbb{R}^S$ such that $h \leq 0$ and $L_{\mapsto s} h \geq h$, then $h_{\mapsto s}^{*} \geq h$.
	\end{proposition}	

	\subsection{Useful Concentration Inequalities}

	Ihe the following, we will repeatedly use the following concentration inequalities.
	\begin{proposition}[Hoeffding Inequality, Theorem 2.8 of \cite{boucheron_concentration_2013}]
		\label{prop:hoeffding}
		Let $(X_i)_{1\leq i \leq n}$ be a collection of independent random variables subject to $i \in {1,\cdots,n}$, $\mathbb{P}(X_i \in [a_i,b_i]) = 1$, and $\mathbb{E}[X_i] = \mu_i$, then with a probability at least $1-\delta$ it holds that 
		\begin{equation}
			\bigg|\sum_{i=1}^n(X_i - \mu_i) \bigg| \leq \sqrt{\frac{1}{2}\sum_{i=1}^{n}(b_i - a_i)^2 \ln \bigg(\frac{2}{\delta}\bigg)}
		\end{equation}
	\end{proposition}
	
	\begin{proposition}[Azuma’s inequality]
		\label{prop:azuma}
		Let $(X_n,\mathcal{F}_n)_{n\in\mathbb{N}}$ be a Martingale difference sequence (MDS) such as $\vert X_n \vert \leq a$ almost sure for all $n\in\mathbb{N}$. Then for all $\delta \in (0,1]$,
		\begin{equation}
			\mathbb{P} \left(\sum_{i=1}^{n} X_{i} \geq a \sqrt{2 n \ln \left(\frac{1}{\delta}\right)}\right) \leq \delta
		\end{equation}
	\end{proposition}

	\begin{proposition}[Emperical Bernstein's Inequality, Theorem 1 of \cite{audibert_explorationexploitation_2009}]
		\label{prop:bernstein}
		Let $(X_i)_{1\leq i \leq n}$ be a collection of independent identical distributed random variables subject to $i \in {1,\cdots,n}$, $\mathbb{P}(X_i \in [a,b]) = 1$, and $\mathbb{E}[X_i] = \mu$, then with a probability at least $1-\delta$ it holds that 
		\begin{equation}
			\bigg|\frac{1}{n}\sum_{i=1}^n(X_i - \mu_i) \bigg| \leq \sqrt{\frac{2 V_{n}(X) \ln (3 / \delta)}{n}}+\frac{3(b-a) \ln (3 / \delta)}{n}
		\end{equation}
		where $V_n(X)$ is the probability variance\footnote{Notably, the probability variance is not an unbiased estimator while the sample variance $V_{n}^{\prime}(X)\defeq \frac{1}{n} \sum_{i=1}^{n-1}\left(X_{i}-\frac{1}{n} \sum_{i=1}^{n} X_{i}\right)^{2}$ is unbiased. } $V_{n}(X)\defeq \frac{1}{n} \sum_{i=1}^{n}\left(X_{i}-\frac{1}{n} \sum_{i=1}^{n} X_{i}\right)^{2}$.
	\end{proposition}
	
	\section{Problem Formulation and VB-UCRL Solution}
	\label{sec:proSolu}
	\subsection{Problem Formulation}
		In the problem setting, we consider the underlying MDP contains both endogeneous and exogeneous uncertainty. Specifically, the mean rewards and transition probabilities depend on the current step $t$, which are denoted as $r_t(s,a)$ and $p_t(s^{\prime}|s,a)$ respectively. Accordingly, the time-heterogeneous MDP at step $t$ can be written as $M_t = <\mathcal{S},\mathcal{A},r_t,p_t,s_1>$. All MDPs $M_t$ are communicating with diameter $D_t \leq D$, where $D$ denotes a common upper bound. 

		Let $M_t$ be the true MDP. We consider the learning problem where $\mathcal{S}$, $\mathcal{A}$ and $r_{\max}$ are known, where reward $r_t$ and transitions $p_t$ are unknown and need to be estimated online. We try to evaluate the performance of a learning algorithm $\mathfrak{A}$ after $T$ time-steps by its cumulative regret
		\begin{equation}
			\label{eq:regret_bound_def}
			\Delta(\mathfrak{A},T) = v^{*,T}(s_1) - \sum_{t=1}^T r_t(s_t,a_t)
		\end{equation} 
		where $v^{*,T}(s_1)$ denotes the optimal $T$-step average reward starting from $s_1$\footnote{Interesting readers could refer to Page 338 of \cite{puterman_markov_1994} for the relationship between $v$ and $h$.}. 
		
		% Notably, this basically corresponds to the standard notion of regret in \cite{jaksch_near-optimal_2010}, since in communicating MDPs, $\sum_{t=1}^T g^{*}_t - v^{*,T}(s_1) $ is bounded by an order of $LDr_{\max}$ if there are at most $L$ changes in the time-heterogeneous MDPs.
		
		We assume that the \emph{variations} in mean rewards and transition probabilities are bounded in the $T$ steps.
		$V^T_r \defeq \sum_{t=1}^{T-1} \max_{s,a} \vert r_{t+1} (s,a) - r_{t} (s,a) \vert$, and $V_T^p \defeq \sum_{t=1}^{T-1} \max_{s,a} \Vert \bar{p}_{t+1} (\cdot|s,a) - \bar{p}_{t} (\cdot|s,a) \Vert_1$. 
	\subsection{Variation-aware Bernstein-based Upper Confidence Reinforcement Learning (VB-UCRL)}
	For RL in the changing MDP settings, we propose the variation-aware Bernstein-based upper confidence reinforcement learning (VB-UCRL), which is a variant of UCRL2 \cite{jaksch_near-optimal_2010}, implements the paradigm of “\textit{optimism in the face of uncertainty}” and constructs MDPs in confidence interval based on the empirical Bernstein inequality in Prop. \ref{prop:bernstein} \cite{audibert_explorationexploitation_2009} rather than the Hoeffding inequality for UCRL2 in Prop. \ref{prop:hoeffding} \cite{boucheron_concentration_2013}. 
	
	VB-UCRL proceeds through episodes $k = 1, 2, \cdots$. Without loss of generality, $t_{k}$ is the starting time of episode $k$. $N_k(s,a)$ is the number of visits in $(s,a)$ before episode $k$. Here, VB-UCRL enters into a new episode $k+1$ after once there exists one state-action pair $(s,a)$ having just been played satisfies $\nu_k(s,a) = N_k^{+}(s,a)$, where $\nu_k(s,a)$ denotes the number of visits to $(s,a)$ in episode $k$ and $N_k^{+}(s,a) = \max\{1,N_k(s,a)\}$. For episode $k+1$, for all state-action pairs, $N_{k+1}(s,a) = N_{k}(s,a) + \nu_k(s,a)$. Besides, $t_k$ is defined as the starting time of episode $k$, that is, $t_{k+1} \defeq \inf\bigg\{T \geq t > t_k: \sum_{\tau =1}{t-1} \mathbbm{1}\{(s_{\tau},a_{\tau}) = (s_t,a_t)\} \geq \max\big\{1,2\mathbbm{1}\{(s_{\tau},a_{\tau}) = (s_t,a_t)\}  \big\}\bigg\} $ and $t_1 =1$.

	At the beginning of each episode $k$, VB-UCRL computes a set $\mathcal{M}$ of statistically plausible MDPs given the observations so far, that is, 
	\begin{align}
		\mathcal{M}_k =& \bigg\{M = <\mathcal{S}, \mathcal{A}, \tilde{r}, \tilde{p}>: \tilde{r}(s,a) \in \mathcal{B}_r^{k}(s,a), \label{eq:extend_m} \\
		&\quad \tilde{p}(s,a) \in \mathcal{B}_p^{k}(s,a), \sum_{s'}\tilde{p}(s'|s,a)=1\bigg\}, \nonumber
	\end{align}
	where $\mathcal{B}_r^{k}$ and $\mathcal{B}_p^{k}$ are high-probability (adapted) confidence intervals on the rewards and transition probabilities of the true MDP $M$. Specifically, 
	\begin{align}
		& \mathcal{B}_p^{k}(s,a,s') =  [0,1]\cap  \left[\hat{p}(s'|s,a) - \beta_{p,k}^{sas'}-\hat{V}_{p},\hat{p}(s'|s,a) + \beta_{p,k}^{sas'}+\hat{V}_{p}\right] \label{eq:extend_p}
	\end{align}
	where $\hat{p}$ is set as an estimate of transitions, that is,
	\begin{equation}
		\hat{p}_k(s,a) = \frac{1}{N_k(s,a)} \sum_{t=1}^{t_k-1} \mathbbm{1}\{(s_t,a_t) = (s,a)\}
	\end{equation}
	and $\hat{V}_{p} \leq V_p^T$ is an estimate over the variations on the transition probabilities. Moreover, using empirical Bernstein inequality \cite{audibert_explorationexploitation_2009}, 
	\begin{align}
		\beta_{p,k}^{sas'} \defeq & 2\sqrt{\frac{\hat{\sigma}_{p,k}^2(s'|s,a)}{N_k^{+}(s,a)}\ln\big(\frac{6SAN_k^{+}(s,a)}{\delta}\big)}  + \frac{6\ln\big(\frac{6SAN_k^{+}(s,a)}{\delta}\big)}{N_k^{+}(s,a)}
	\end{align}
	where $\delta \in (0,1) $. Similarly,
	\begin{equation}
		\label{eq:extend_r}
		\mathcal{B}_r^{k}(s,a) = \left[\hat{r}(s,a) - \beta_{r,k}^{sa} - \hat{V}_r,\hat{r}(s,a) + \beta_{r,k}^{sa} + \hat{V}_r \right] \cap [0,r_{\max}]
	\end{equation}
	where 
	\begin{align}
		\beta_{r,k}^{sa} \defeq & 2\sqrt{\frac{\hat{\sigma}_{r,k}^2(s,a)}{N_k^{+}(s,a)}\ln\big(\frac{6SAN_k^{+}(s,a)}{\delta}\big)}  + \frac{6r_{\max}\ln\big(\frac{6SAN_k^{+}(s,a)}{\delta}\big)}{N_k^{+}(s,a)}
	\end{align}
	and $\hat{V}_r$ is an estimate over the variations on the mean rewards. 
	$\hat{r}_k$ is the empirical average of rewards, namely
	\begin{equation}
		\hat{r}_k(s,a) = \frac{1}{N_k(s,a)} \sum_{t=1}^{t_k-1} \mathbbm{1}\{(s_t,a_t) = (s,a)\}\cdot r_t
	\end{equation}
	The estimated transition probability $\hat{p}_k(s'|s,a)$ corresponds to the sample mean of independent identical Bernouilli random variable with mean $p(s'|s,a)$ and the population variance can be approximately computed as $\hat{\sigma}_{p,k}^2(s'|s,a) = \hat{p}_k(s'|s,a) (1-\hat{p}_k(s'|s,a))$. The population variance of the reward can be computed recursively at the end of every episode as 
	\begin{align}
		\hat{\sigma}_{r, k+1}^{2}(s, a) &\defeq\frac{1}{N_{k+1}^{+}(s, a)}\left(\sum_{l=1}^{k} S_{l}(s, a)\right)-\left(\hat{r}_{k+1}(s, a)\right)^{2} \nonumber \\
		&=\frac{N_{k}(s, a)}{N_{k+1}^{+}(s, a)}\left(\hat{\sigma}_{r, k}^{2}(s, a)+\left(\hat{r}_{k}(s, a)\right)^{2}\right)  + \frac{S_{k}(s, a)}{N_{k+1}^{+}(s, a)} -\left(\hat{r}_{k+1}(s, a)\right)^{2} 
	\end{align}
	where $S_{k}(s, a) \defeq \sum_{t=1}^{t_k-1}\mathbbm{1}\{(s_t,a_t) = (s,a)\}\cdot r_t^2$. As pointed out by Section 3.1.1 of \cite{jaksch_near-optimal_2010}, any bounded parameter MDP can be equivalently represented by an \textit{extended MDP} $\tilde{M}_k^{+}$, by combining all plausible MDPs constructed above into a single MDP with identical state space $\mathcal{S}$ but with an extended compact action space $\mathcal{A}^{+} = \cup_{a\in \mathcal{A}_s} {a} \times B_r(s,a) \times B_p (s,a)$.

	Moreover, the extended greedy operator is defined as 
	\begin{equation}
		G_{k} v(s) \in \underset{a \in \mathcal{A}_{s}}{\max }\left\{\max _{r \in B_{r}^{k}(s, a)} r+\max _{p \in B_{p}^{k}(s, a)} p^{T} v\right\}
	\end{equation}

	Afterwards, VB-UCRL chooses an optimistic MDP $M_k$ (with respect to the achievable average reward) among these plausible MDPs $\mathcal{M}_k$, and executes a policy $\pi_k$  which is (nearly) optimal for the optimistic MDP $M_k$, that is,
	\begin{equation}
		\label{eq:optimal_policy_extend}
		\max _{\pi \in \Pi^{\mathrm{SD}}}\left\{\sup _{M^{\prime} \in \mathcal{M}_{k}} g_{M^{\prime}}^{\pi}\right\}=\sup _{M^{\prime} \in \mathcal{M}_{k}}\left\{\max _{\pi \in \Pi^{\mathrm{SD}}} g_{M^{\prime}}^{\pi}\right\}=\sup _{M^{\prime} \in \mathcal{M}_{k}} g_{M^{\prime}}^{*}
	\end{equation}
	
	\begin{proposition}[Proposition 2.7 of \cite{fruit_exploration-exploitation_2019}]
		\label{prop:evi_gap}
		Consider the gain $g$ and bias $h$ returned by Alg. \ref{al:relative_value-iteration}. Under the same assumptions as Prop. \ref{prop:value_iter_converge}, $|g - g^{*}| \leq \epsilon/2$ and for all $s \in \mathcal{S}$, $|Lh(s)-h(s)-g|\leq \epsilon$, where $\epsilon \in (0,r_{\max})$ is the accuracy given as input of Alg. \ref{al:relative_value-iteration}.
	\end{proposition}
	
	By Prop. \ref{prop:evi_gap}, if we run extended value iteration in Alg. \ref{al:relative_value-iteration} on $M_{k}^{+}$ with accuracy $\epsilon_k = r_{\max}/t_k$, we have that
	\begin{equation}
		|g_k - g_k^{*}| \leq \epsilon_k/2 = \frac{r_{\max}}{2t_k}
	\end{equation}
	and 
	\begin{equation}
		\Vert L^{\alpha}_{k} h_k - h_k - g_k e\Vert_{\infty}\leq \epsilon_k = \frac{r_{\max}}{t_k}
	\end{equation}
	where $(g_k,h_k,\pi_k) = EVI(L_k^{\alpha},G_k^{\alpha},\frac{r_{\max}}{t_k},0,s_1)$. $r_k$ and $p_k$ are denoted as the optimistic reward and transitions at episode $k$.

	We first summarize the VB-UCRL without variation-aware restarts as Alg. \ref{al:vb-ucrl}.
	\begin{algorithm}
		\caption{VB-UCRL without Variation-Aware Restarts}
		\label{al:vb-ucrl}
		\hspace*{\algorithmicindent} \textbf{Input}: Confidence $\delta \in (0,1)$, $r_{\max}$, $\mathcal{S}$, $\mathcal{A}^{+}$.
		\begin{algorithmic}[1]
			\State Initialize $t\defeq 1$ and observe $s_{1}$ and for any $\left(s, a, s^{\prime}\right) \in \mathcal{S} \times \mathcal{A} \times \mathcal{S}$: $N_{1}(s, a)=0$, $\hat{p}_{1}\left(s^{\prime} | s, a\right)=0$,
			$\hat{r}_{1}(s, a)=0$, $\hat{\sigma}_{p, 1}^{2}\left(s^{\prime} | s, a\right)=0$, $\hat{\sigma}_{r, 1}^{2}(s, a)=0$.
			\For{episodes $k=1,2, \ldots$} 
			\State Set $t_{k} \leftarrow t$ and episode counters $\nu_{k}(s, a) \leftarrow 0$.
			\State Compute the upper-confidence bounds \eqref{eq:extend_p} and \eqref{eq:extend_r} and the extended MDP $\mathcal{M}_{k}^{+}$ as in \eqref{eq:extend_m}.
			\State Compute an $r_{\max } / t_{k}$-approximation $\pi_{k}$ of \eqref{eq:optimal_policy_extend} $(g_k,h_k,\pi_k) = EVI\left(\mathcal{L}^{\alpha}_{k}, \mathcal{G}^{\alpha}_{k}, \frac{r_{\max}}{t_{k}}, 0, s_{1}\right)$.
			\State Sample action $a_{t} \sim \pi_{k}\left(\cdot | s_{t}\right)$.
			\While{ $t_{k}=t$ or $\nu_{k}\left(s_{t}, a_{t}\right) \leq \max \left\{1, N_{k}\left(s_{t}, a_{t}\right)\right\}$ }
			\State Execute $a_{t},$ obtain reward $r_{t},$ and observe $s_{t+1}$.
			\State Sample action $a_{t+1} \sim \pi_{k}\left(\cdot | s_{t+1}\right)$.
			\State $\operatorname{Set} \nu_{k}\left(s_{t}, a_{t}\right) \leftarrow \nu_{k}\left(s_{t}, a_{t}\right)+1$ and set $t \leftarrow t+1$.
			\EndWhile 
			\State $\operatorname{Set} N_{k+1}(s, a) \leftarrow N_{k}(s, a)+\nu_{k}(s, a)$.
			\State Update statistics (i.e., $\left(\hat{p}_{k+1}, \hat{r}_{k+1}, \hat{\sigma}_{p, k+1}^{2}, \hat{\sigma}_{r, k+1}^{2}\right)$.
			\EndFor
		\end{algorithmic}
	\end{algorithm}
	
	Next, we can formally give VB-UCRL in Alg. \ref{al:vb-ucrl-restart}. In particular, we restart Alg. \ref{al:vb-ucrl} in phases by continuously tuning the confidence parameter $\frac{\delta}{2t^2}$ according to a schedule dependent on the variations.
	\begin{algorithm}
		\caption{VB-UCRL with Restarts}
		\label{al:vb-ucrl-restart}
		\hspace*{\algorithmicindent} \textbf{Input}: State space $\mathcal{S}$, action space $\mathcal{A}$, confidence parameter $\delta$, variation terms $V_r^T$ and $V_p^T$.
		\begin{algorithmic}[1]
			\State Initialization: Set current time step $\tau \defeq 1$.
			\For{phase $i=1,2, \ldots$} 
			\State Perform VB-UCRL in Algorithm \ref{al:vb-ucrl} with confidence parameter $\delta/2{\tau}^2$ for $\theta_i \defeq \lceil \frac{i^2}{(2V_r^T + V_p^T)^2}\rceil$ steps.
			\State Update $\tau \leftarrow \tau + \theta_i $.
			\EndFor
		\end{algorithmic}
	\end{algorithm}

	\section{Fundamental Limit of VB-UCRL}
	\label{sec:result}
	In Section, we first derive the upper regret bound of VB-UCRL without variation-aware restarts and then extend it to VB-UCRL with restarts.
	\subsection{Upper Regret Bound of VB-UCRL without Variation-Aware Restarts}
	The following theorem gives the limits of regret bound in \eqref{eq:regret_bound_def} for VB-UCRL without variation-aware restarts\footnote{For simplicity of representation, in this part, we slightly abuse the notations for VB-UCRL with and without variation-aware restarts.}.
	\begin{theorem}
		\label{thm:regret_bound_vb-ucrl}
		There exists a numerical constant $\beta > 0 $ such that for any communicating MDP, if $\hat{V}_p$ and $\hat{V}_r$ are set as the true values $V_p^T$ and $V_r^T$, with probability at least $1 - \delta$, it holds that for all initial state distributions $\nu_1 \in \Delta_s$ ($\Delta_s$ denotes a $S$-dimensional simplex.) and for all time horizons $T \geq SA$ 
		\begin{align}
			& \Delta(\text{VB-UCRL}, T) \leq  \label{eq:regret_bound_vb-ucrl}\\
			& \max(r_{\max},Dr_{\max}) \bigg(86 \sqrt{T\ln\big(\frac{T}{\delta}\big) \sum_{s,a}\Gamma(s,a)}   + 144 S^2A \ln\big(\frac{T}{\delta}\big) \ln(T) \bigg) +  Dr_{\max}T V_p^T  +  2TV_r^T \nonumber
		\end{align}
		where $\Gamma(s,a) \defeq \Vert p(\cdot|s,a) \Vert_0 = \sum_{s'\in \mathcal{S}} \mathbbm{1}\{p(s'|s,a) >0\}$ and $\Gamma \defeq \max_{s,a\in\mathcal{S}\times\mathcal{A}}\Gamma(s,a)$.
	\end{theorem}
	\begin{proof}
		By Lemma \ref{lem:hete_tstep_average}, which shows that under the event where the true MDP falls into the scope of plausible MDPs, the $T$-step reward in the changing MDP settings could be bounded by the optimistic average reward $g^{*}$. So, we have
		\begin{align}
			& \Delta(\text{VB-UCRL},T) \nonumber \\
			= & v^{*,T}(s_1) - \sum_{t=1}^T r_t(s_t,a_t) \\
			\leq & \sum_{t=1}^T \big( g^{*}  - r_t(s_t,a_t))\big) + Dr_{\max} \nonumber
		\end{align}
		where $g^{*} \defeq \min_k g^{*}_k $, and $g^{*}_k \defeq \max_{\pi,M\in \mathcal{M}_k}g^{*}_k(M)$.
	
		By Lemma \ref{lem:r_t_azuma}, which can be interpreted as removing all the randomness due to the stochasticity of the observed rewards and the executed policy, at the expense of $\tilde{\mathcal{O}}(\sqrt{T})$, with a probability at least $1-\frac{\delta}{6}$, $\Delta(\text{VB-UCRL},T)$ could be rewritten as 
		\begin{align}
			& \Delta(\text{VB-UCRL},T) \leq  \sum_{t=1}^T \big(g^{*} - r_t(s_t,a_t)\big) \nonumber \\
			& \leq \sum_{t=1}^T \bigg(g^{*} - \sum_{a\in A_{s_t}} \pi_{k_t}(s_t,a)r(s_t,a) \bigg) + 2r_{\max} \sqrt{T\ln \big(\frac{4T}{\delta}\big)} \label{eq:bound_ucrlb} \\
			& = \sum_{t=1}^{k_T} \sum_{s\in \mathcal{S}} \nu_k(s) \bigg(g^{*} - \sum_{a\in A_{s}} \pi_{k}(a|s)r(s,a) \bigg)  + 2r_{\max} \sqrt{T\ln \big(\frac{4T}{\delta}\big)} \nonumber\\
			& \overset{(a)}{=} \sum_{t=1}^{k_T} \Delta_k + 2r_{\max} \sqrt{T\ln \big(\frac{4T}{\delta}\big)}\nonumber
		\end{align}
		where the equation $(a)$ comes after $ \Delta_k \defeq \sum_{s\in \mathcal{S}} \nu_k(s) \bigg(g^{*} - \sum_{a\in A_{s}} \pi_{k}(a|s)r(s,a) \bigg)$, and $k_t \defeq \sup\{k \geq 1: t \geq t_k\}$ denotes the integer-valued random variable indexing the current episode at time $t$. By Prof. \ref{prop:bound_k_t}, $k_T \leq SA \log_2 \big( \frac{8T}{SA}\big)$ is bounded for $T \geq SA$.

		Next, we derive the bound for $\Delta_k$ with a high probability. By Lemma \ref{lem:delta_k_decompose_p_r}, if the true MDP falls into the scope of plausible MDPs ($M \in \mathcal{M}_k, \forall k$), $\Delta_k$ could be upper bounded by 
		\begin{equation}
			\Delta_k \leq \Delta_k^{p} + \Delta_k^{r} + \frac{3\epsilon_k}{2} \sum_{s\in\mathcal{S}}\nu_k(s)\label{eq:delta_k_p_r}
		\end{equation}
		where 
		\begin{displaymath}
			\Delta_k^{p} \defeq \alpha \sum\limits_{s \in \mathcal{S}} \nu_{k}(s)\bigg(\sum\limits_{\substack{a \in \mathcal{A}_{s}\\ s^{\prime} \in \mathcal{S}}} \pi_{k}(a | s) p_{k}\big(s^{\prime} | s, a\big) h_{k}\big(s^{\prime}\big)-h_{k}(s)\bigg)
		\end{displaymath}
		and
		\begin{displaymath}
			\Delta_{k}^{r} \defeq \sum\limits_{s \in \mathcal{S}} \sum\limits_{a \in \mathcal{A}_{s}} \nu_{k}(s) \pi_{k}(a | s)\big(r_{k}(s, a)-r(s, a)\big).
		\end{displaymath}
		
		We further decompose $\Delta_k^{p}$ into two parts $\Delta_k^{p} = \Delta_k^{p1} + \Delta_k^{p2}$, where 
		\begin{displaymath}
			\Delta_k^{p1} \defeq \alpha \sum\limits_{s,a,s'} \nu_{k}(s) \pi_{k}(a | s) \bigg(p_{k}\big(s^{\prime} | s, a\big) - p\big(s^{\prime} | s, a\big) \bigg) h_{k}\big(s^{\prime}\big) 
		\end{displaymath}
		and 
		\begin{displaymath}
			\Delta_k^{p2} \defeq \alpha \sum\limits_{s} \nu_{k}(s)\bigg(\sum\limits_{\substack{a,s^{\prime}}} \pi_{k}(a | s) p\big(s^{\prime} | s, a\big) h_{k}\big(s^{\prime}\big)-h_{k}(s)\bigg),
		\end{displaymath}
		and bound them in Lemma \ref{lem:delta_p1_p3}, Lemma \ref{lem:delta_p3}, and Lemma \ref{lem:delta_p2}. Accordingly, with a probability $1-\frac{\delta}{3}$, 
		\begin{align}
			\sum\limits_{k=1}^{k_T} \Delta_k^{p} 
			& \leq  Dr_{\max} \sum\limits_{k=1}^{k_T} \sum\limits_{s,a} \nu_{k}(s,a) (\beta_{p,k}^{sa} + V_p^T ) +  6 D r_{\max} \sqrt{T \ln\big(\frac{6T}{\delta}\big)}  + k_T D r_{\max} \label{eq:bound_p} \\
			& \leq  Dr_{\max} \sum\limits_{k=1}^{k_T} \sum\limits_{s,a} \nu_{k}(s,a) \beta_{p,k}^{sa} +   Dr_{\max}T V_p^T +  6 D r_{\max} \sqrt{T \ln\big(\frac{6T}{\delta}\big)}  + k_T D r_{\max} \nonumber 
		\end{align}
		where $\beta_{p,k}^{sa} \defeq \sum_{s^{\prime}}\beta_{p,k}^{sas^{\prime}}$. Similarly, by Lemma \ref{lem:delta_r}, with probability at least $1-\frac{\delta}{6}$, we have 
		\begin{align}
			\sum\limits_{k=1}^{k_T} \Delta_k^{r}
			\leq  4 r_{\max} \sqrt{T \ln\big(\frac{4T}{\delta}\big)} + 2 \sum\limits_{k=1}^{k_T} \sum\limits_{s,a} \nu_{k}(s,a) \beta_{r,k}^{sa} + 2TV_r^T \label{eq:bound_r} 
		\end{align}

		As proved in Thm \ref{thm:mdp_evi_true_range}, the event that the true MDP falls into the scope of plausible MDPs occurs with a probability at least $1-\frac{\delta}{3}$. Merging \eqref{eq:delta_k_p_r}, \eqref{eq:bound_p}, \eqref{eq:bound_r} into \eqref{eq:bound_ucrlb}, with a probability at least $1-\frac{5\delta}{6}$, for all $T \geq SA$, we have \eqref{eq:bound_ucrlb_comb} as
		\begin{align}
			& \Delta(\text{VB-UCRL},T) \nonumber \\
			& \leq 2r_{\max} \sqrt{T\ln \big(\frac{4T}{\delta}\big)} +  \sum\limits_{k=1}^{k_T} \frac{3\epsilon_k}{2} \sum_{s}\nu_k(s) + Dr_{\max} \sum\limits_{k=1}^{k_T} \sum\limits_{s,a} \nu_{k}(s,a) \beta_{p,k}^{sa} + Dr_{\max}T V_p^T\\
			&\quad + 6 D r_{\max} \sqrt{T \ln\big(\frac{6T}{\delta}\big)} + k_T D r_{\max} \nonumber \\
			&\quad  + 4 r_{\max} \sqrt{T \ln\big(\frac{4T}{\delta}\big)} + 2 \sum\limits_{k=1}^{k_T} \sum\limits_{s,a} \nu_{k}(s,a) \beta_{r,k}^{sa} + 2TV_r^T \label{eq:bound_ucrlb_comb} \\
			& \leq 6r_{\max} \sqrt{T\ln \big(\frac{4T}{\delta}\big)} + 6 D r_{\max} \sqrt{T \ln\big(\frac{6T}{\delta}\big)} + D r_{\max} SA \log_2\big(\frac{T}{SA}\big) + Dr_{\max}T V_p^T  +  2TV_r^T   \nonumber \\
			&\quad +  r_{\max}\sum\limits_{k=1}^{k_T} \frac{3}{2t_k} \sum_{s}\nu_k(s) + 2 \sum\limits_{k=1}^{k_T} \sum\limits_{s,a} \nu_{k}(s,a) \beta_{r,k}^{sa} + Dr_{\max} \sum\limits_{k=1}^{k_T} \sum\limits_{s,a} \nu_{k}(s,a) \beta_{p,k}^{sa} \nonumber 
		\end{align}

		As for the last three terms in \eqref{eq:bound_ucrlb_comb}, we have 
		\begin{itemize}
			\item Since $t_k \geq N_k^{+}(s,a)$ for all $(s,a)$, we have
			\begin{align}
				& r_{\max}\sum\limits_{k=1}^{k_T} \frac{3}{2t_k} \sum_{s}\nu_k(s) 
				= \frac{3r_{\max}}{2} \sum_{s,a} \sum\limits_{k=1}^{k_T} \frac{\nu_k(s,a)}{t_k} \nonumber \\
				\leq & \frac{3r_{\max}}{2} \sum_{s,a} \sum\limits_{k=1}^{k_T} \frac{\nu_k(s,a)}{N_k^{+}(s,a)} \nonumber\\
				\stackrel{(a)}{\leq} & \frac{3r_{\max}}{2} \sum_{s,a} 2 + 2 \ln \big(N_{k_T+1}^{+}(s,a)\big) \\
				\stackrel{(b)}{\leq} & \frac{3r_{\max}}{2} SA \bigg(2 + 2 \ln \big(\frac{\sum_{s,a} N_{k_T+1}^{+}(s,a)}{SA}\big)\bigg)\nonumber\\
				\stackrel{(c)}{\leq} & \frac{3r_{\max}SA}{2}  \bigg(2 + 2 \ln \big(\frac{T}{SA}\big)\bigg) \nonumber\\
				\leq & 3r_{\max}SA \bigg(1 +  \ln T\bigg) \nonumber
			\end{align}
			where the equation $(a)$ comes from Prop. \ref{prop:sum_to_ln}, while the inequality $(b)$ leverages the concavity of a logarithmic function and the Jensen inequality. The equation $(c)$ is due to that $\sum_{s,a} N_{k+1}^{+}(s,a) \leq T$.
			\item Taking account of the definition of $\beta_{r,k}^{sa}$,
			\begin{align}
				& 2 \sum\limits_{k=1}^{k_T} \sum\limits_{s,a} \nu_{k}(s,a) \beta_{r,k}^{sa} \nonumber \\
				= & 4 \sum\limits_{k=1}^{k_T} \sum\limits_{s,a} \bigg[ \nu_{k}(s,a) \sqrt{\frac{\hat{\sigma}_{r,k}^2(s,a)}{N_k^{+}(s,a)}\ln\big(\frac{6SAN_k^{+}(s,a)}{\delta}\big)} + 3r_{\max}\ln\big(\frac{6SAN_k^{+}(s,a)}{\delta}\big)\frac{\nu_{k}(s,a)}{N_k^{+}(s,a)}\bigg] \nonumber  \\
				\stackrel{(a)}{\leq} & 4 r_{\max} \sqrt{\ln\big(\frac{6SAT} {\delta}\big)} \sum\limits_{k=1}^{k_T} \sum\limits_{s,a} \bigg[  \frac{\nu_{k}(s,a)}{\sqrt{N_k^{+}(s,a)}} \bigg] + 12r_{\max} \ln\big(\frac{6SAT} {\delta}\big)  \sum\limits_{k=1}^{k_T} \sum\limits_{s,a} \frac{\nu_{k}(s,a)}{N_k^{+}(s,a)}\nonumber \\
				\stackrel{(b)}{\leq} & 4 r_{\max} \sqrt{\ln\big(\frac{6SAT} {\delta}\big)} \sum\limits_{k=1}^{k_T} \sum\limits_{s,a} \bigg[  \frac{\nu_{k}(s,a)}{\sqrt{N_k^{+}(s,a)}} \bigg] + 12r_{\max} \ln\big(\frac{6SAT} {\delta}\big)  \sum\limits_{k=1}^{k_T} \sum\limits_{s,a} \frac{\nu_{k}(s,a)}{N_k^{+}(s,a)}\nonumber \\
				\stackrel{(c)}{\leq} & 12 r_{\max} \sqrt{SAT \ln\big(\frac{6SAT} {\delta}\big)}  + 24 r_{\max} SA \ln\big(\frac{6SAT} {\delta}\big)  (1 + \ln T)\nonumber 
			\end{align}
			where the equation $(a)$ comes from $\hat{\sigma}_{r,k}^2(s,a) \leq r_{\max}^2$ and $\ln\big(\frac{6SAN_k^{+}(s,a)} {\delta}\big) \leq \ln\big(\frac{6SAT} {\delta}\big)$, the inequality $(b)$ comes from Prop. \ref{prop:sum_to_ln}, and the inequality $(c)$ comes from similar deduction as the previous term. 
			\item Similarly, by the definition of $\beta_{p,k}^{sa}$,
			\begin{align}
				& Dr_{\max} \sum\limits_{k=1}^{k_T} \sum\limits_{s,a} \nu_{k}(s,a) \beta_{p,k}^{sa} \nonumber\\
				\leq & 
				2Dr_{\max} \sqrt{\ln \left(\frac{6 S A T}{\delta}\right)} \sum_{s, a} \sum_{k=1}^{k_{T}} \frac{\nu_{k}(s, a)}{\sqrt{N_{k}^{+}(s, a)}}  \cdot \sum_{s^{\prime} \in S} \sqrt{\widehat{p}_{k}\left(s^{\prime} \mid s, a\right)\left(1-\hat{p}_{k}\left(s^{\prime} \mid s, a\right)\right)} \nonumber\\
				&+6Dr_{\max} S \ln \left(\frac{6 S A T}{\delta}\right) \sum_{s, a} \sum_{k=1}^{k_{T}} \frac{\nu_{k}(s, a)}{N_{t}^{+}(s, a)}\nonumber\\
				\stackrel{(a)}{\leq} & 6Dr_{\max} \sqrt{\ln \left(\frac{6 S A T}{\delta}\right)} \sum_{s, a}  \sqrt{(\Gamma(s,a)-1) N_{k_T+1}(s,a)} \nonumber\\
				&+12 Dr_{\max} S^2 A \ln \left(\frac{6 S A T}{\delta}\right) (1 + \ln T)\\
				\stackrel{(b)}{\leq} & 6Dr_{\max} \sqrt{\ln \left(\frac{6 S A T}{\delta}\right)}  \sqrt{ \sum_{s, a} \Gamma(s,a)  \sum_{s, a} N_{k_T+1}(s,a)} \nonumber\\
				&+12 Dr_{\max} S^2 A \ln \left(\frac{6 S A T}{\delta}\right) (1 + \ln T)\nonumber\\
				\leq & 6Dr_{\max} \sqrt{\ln \left(\frac{6 S A T}{\delta}\right)}  \sqrt{\left( \sum_{s, a} \Gamma(s,a)\right)  T } \nonumber\\
				&+12 Dr_{\max} S^2 A \ln \left(\frac{6 S A T}{\delta}\right) (1 + \ln T)\nonumber
			\end{align}
			where the inequality $(a)$ comes from Prop. \ref{prop:sum_to_ln} while the inequality $(b)$ comes from Cauchy-Schwartz inequality.
 		\end{itemize}
		In summary, \eqref{eq:bound_ucrlb_comb} could be written as \eqref{eq:bound_ucrlb_comb_f}.
		\begin{align}
			& \Delta(\text{VB-UCRL},T) \nonumber \\
			& \leq 6r_{\max} \sqrt{T\ln \big(\frac{4T}{\delta}\big)} + 6 D r_{\max} \sqrt{T \ln\big(\frac{6T}{\delta}\big)} + D r_{\max} SA \log_2\big(\frac{8T}{SA}\big) + Dr_{\max}T V_p^T  +  2TV_r^T   \nonumber \\
			& \quad +  3r_{\max}SA \bigg(1 +  \ln T\bigg) +  12 r_{\max} \sqrt{SAT \ln\big(\frac{6SAT} {\delta}\big)} + 24 r_{\max} SA \ln\big(\frac{6SAT} {\delta}\big)  (1 + \ln T)  \label{eq:bound_ucrlb_comb_f} \\
			&\quad + 6Dr_{\max} \sqrt{\ln \left(\frac{6 S A T}{\delta}\right)}  \sqrt{\left( \sum_{s, a} \Gamma(s,a)\right)  T } +12 Dr_{\max} S^2 A \ln \left(\frac{6 S A T}{\delta}\right) (1 + \ln T) \nonumber
		\end{align}

		By Prop. \ref{prop:bound_ucrlb_comb_s}, \eqref{eq:bound_ucrlb_comb_f} could be further simplified as
		\begin{align}
			& \Delta(\text{VB-UCRL}, T)  \nonumber\\
			\leq & \max(r_{\max},Dr_{\max}) \bigg(43 \sqrt{T\ln\big(\frac{T}{\delta}\big) \sum_{s,a}\Gamma(s,a)}  + 72 S^2A \ln\big(\frac{T}{\delta}\big) \ln(T) \bigg) +  Dr_{\max}T V_p^T  +  2TV_r^T \nonumber
		\end{align}
	\end{proof}
	
	\subsection{Upper Regret Bound of VB-UCRL}

	\begin{theorem}
		After any $T$ steps, the regret of VB-UCRL with restarting in Algorithm \ref{al:vb-ucrl-restart} is bounded by 
		\begin{align}
			& 155 \max(r_{\max},Dr_{\max}) (V_r^T + V_p^T)^{1/3} T^{2/3}   \cdot \sqrt{\ln\big(\frac{2T^3}{\delta}\big) \sum_{s,a}\Gamma(s,a)} \\
			& \ + 144 \max(r_{\max},Dr_{\max})  S^2A \ln\big(\frac{2T^3}{\delta}\big) \ln(2T^3)  \nonumber 
		\end{align} 
		\label{thm:vb-ucrl-restart}
	\end{theorem}
	\begin{proof}
		Inspired by the proof of Theorem 2 of \cite{gajane_variational_2019}, we write $V_r^{(i)}$ and $V_p^{(i)}$ for the variation of rewards and transition probabilities in Phase $i$ and abbreviate $V^{(i)} \defeq V_r^{(i)} + 2V_p^{(i)}$, $V \defeq 2V_r + V_p$ and $\theta_i \defeq \lceil \frac{i^2}{V^2}\rceil$.

		If the number of phases up to $T$ is $N$. We have
		\begin{equation}
			\sum_{i=1}^{N-1} \lceil \frac{i^2}{V^2}\rceil < T \leq \sum_{i=1}^{N} \lceil \frac{i^2}{V^2}\rceil \label{eq:phase_decom}
		\end{equation}

		Recalling that $\sum_{i=1}^N i^2 = \frac{1}{6} N(N+1)(2N+1) > \frac{1}{3}N^3$, we have 
		\begin{equation}
			T > \sum_{i=1}^{N-1} \lceil \frac{i^2}{V^2}\rceil > \sum_{i=1}^{N-1} \frac{i^2}{V^2} > \frac{(N-1)^3}{3V^2}
		\end{equation}
		In other words, $N < 1 + \sqrt[3]{3V^2T}$.

		Denoting $\tau_i$ as the initial step of phase $i$ and $s_{\tau_i}$ as the state visited by the optimal T-step policy at step $\tau_i$, we can decompose the regret as 
		\begin{align}
			\Delta(\text{VB-UCRL}, T)  & = v_{T}^{*}\left(s_{1}\right)-\sum_{t=1}^{T} r_{t} \nonumber \\
			& =\sum_{i=1}^{N}\left(\mathbb{E}\left[v_{\theta_{i}}^{*}\left(s_{\tau_{i}} \right)\right]-\sum_{t=\tau_{i}}^{\tau_{i}-1} r_{t}\right) 
		\end{align}

		By Theorem \ref{thm:regret_bound_vb-ucrl} and a union bound over all possible values for state $s_{\tau_i}$, the $i$-th summand ($i = 1, \cdots, N$) in \eqref{eq:regret_bound_vb-ucrl} with probability $1 - \frac{\delta}{2(\tau^{i})^2}$ is bounded by
		\begin{align}
			& \max(r_{\max},Dr_{\max}) \bigg(86 \sqrt{\ln\big(\frac{2T^3}{\delta}\big) \sum_{s,a}\Gamma(s,a)} \cdot \sqrt{\theta_i} \nonumber \\
			& \quad + 144 S^2A \ln\big(\frac{2T^3}{\delta}\big) \ln(2T^3) \bigg) +  Dr_{\max} V^{(i)} \theta_i \nonumber 
		\end{align}
		
		If $\sqrt[3]{3V^2T}<1$, we have $3V^2T<1$ and hence $3V^2T^2 < T$ and $VT < \sqrt{3}VT < \sqrt{T}$. Furthermore, in this case $N=1$ with $\theta_1 = T$ and $V^{(1)} = V$, so the regret bound is 
		\begin{align}
			& \max(r_{\max},Dr_{\max}) \bigg(86 \sqrt{\ln\big(\frac{2T^3}{\delta}\big) \sum_{s,a}\Gamma(s,a)} \cdot \sqrt{T}  \nonumber \\
			& \quad + 144 S^2A \ln\big(\frac{2T^3}{\delta}\big) \ln(2T^3) \bigg)  +  Dr_{\max} V T  \nonumber \\
			= & 86 \max(r_{\max},Dr_{\max})  \sqrt{\ln\big(\frac{2T^3}{\delta}\big) \sum_{s,a}\Gamma(s,a)} \cdot \sqrt{T}  \nonumber \\
			& \ +  Dr_{\max} V T  + 144 \max(r_{\max},Dr_{\max})   S^2A \ln\big(\frac{2T^3}{\delta}\big) \ln(2T^3)   \nonumber  
		\end{align}
		which is upper bounded by the claimed regret bound.

		On the other hand, if $\sqrt[3]{3V^2T} \geq 1$, then $N < 2\sqrt[3]{3V^2T}$ and summing over all $N$ phases yields from \eqref{eq:phase_decom} that with a probability $\sum_i \frac{\delta}{2(\tau^{(i)})^2} < \sum_t \frac{\delta}{2t^2} < \delta$, the regret is bounded by 
		\begin{align}
			& \max(r_{\max},Dr_{\max}) \bigg(86 \sqrt{\ln\big(\frac{2T^3}{\delta}\big) \sum_{s,a}\Gamma(s,a)} \cdot \sum_{i=1}^{N}\sqrt{\theta_i} \nonumber \\
			& + 144 S^2A \ln\big(\frac{2T^3}{\delta}\big) \ln(2T^3) \bigg) +  Dr_{\max} \sum_{i=1}^N V^{(i)} (\frac{i^2}{V^2} + 1) \nonumber 
		\end{align}
		Noting that using Jensen's inequality 
		$\sum_{i=1}^{N} \sqrt{\theta_{i}} \leq \sqrt{N T} \leq 1.7 \cdot V^{1 / 3} T^{2 / 3}$ and $\sum_{i=1}^N V^{(i)} (\frac{i^2}{V^2} + 1) \leq \sum_{i=1}^N V^{(i)} (\frac{N^2}{V^2} + 1) \leq \frac{N^2}{V} + V < 8.33 V^{1/3}T^{2/3}+V$, we have the bound.
	\end{proof}
	
	\section{Discussions}
	\label{sec:discussion}

	To our best knowledge, \cite{gajane_variational_2019} and \cite{cheung_reinforcement_2020} give the closest regret bound of RL in MDP with both endogeneous and exogeneous uncertainty. In particular, \cite{gajane_variational_2019} shows that if we ignore logarithmic terms (i.e.,regarding the logarithmic terms as a constant), up to a multiplicative numerical constant, the regret bound of variation-aware UCRL in \cite{gajane_variational_2019} is bounded by 
	\begin{displaymath}
		\tilde{O} \left( D r_{\max} (V_r^T + V_p^T)^{1/3} T^{2/3} S\sqrt{A} \right).
	\end{displaymath}
	Meanwhile, \cite{cheung_reinforcement_2020} shows that a regret bound of
	\begin{displaymath}
		\tilde{O} \left(D r_{\max} (V_r^T + V_p^T)^{1/4} S^{2/3} A^{1/2} T^{3/4}\right).
	\end{displaymath}
	
	Similarly, taking account of $\sum_{s,a}\Gamma(s,a) \leq \Gamma S A$, the regret bound of VB-UCRL could be 
	\begin{displaymath}
		\tilde{O}\left(Dr_{\max}(V_r^T + V_p^T)^{1/3} T^{2/3} \sqrt{\Gamma S A} \right)	
	\end{displaymath}
	as in Theorem \ref{thm:regret_bound_vb-ucrl}. Since by definition $\Gamma \leq S$, the regret bound of VB-UCRL is no greater than that in \cite{gajane_variational_2019}. However, as $\Gamma$ is usually equal to $\mathcal{O}(1)$ and significantly smaller than $S$, our bound is superior than \cite{gajane_variational_2019} and \cite{cheung_reinforcement_2020}. In particular, it can save at most $\sqrt{S}$ than \cite{gajane_variational_2019} and $S^{\frac{1}{6}}T^{\frac{1}{12}}$ than \cite{cheung_reinforcement_2020}, respectively.

	\section{Conclusion}
	\label{sec:conclusion}
	In this paper, we studied the problem of online RL for MDP with both endogeneous and exogeneous uncertainty, where the unknown reward and state transition distributions vary within some variation budgets. We first proposed a variation-aware Bernstein-based upper confidence reinforcement learning algorithm. In particular, we allowed UCRL to restart according to a schedule based on the variations, and replaced the commonly used Hoeffding inequality by Bernstein inequality. Our approach achieved tighter regret bounds than those in the literature. Given the wide application of RL, our approach could contribute to the understanding of RL-based optimization performance.
	\section*{Appendix}
	\subsection{Relationship between the True MDP and the Extended MDP}
	We use the Theorem 3.1 of \cite{fruit_exploration-exploitation_2019} to establish the relationship between the true MDP and the extended MDP.

	\begin{theorem}[Theorem 3.1 of \cite{fruit_exploration-exploitation_2019}]
		\label{thm:mdp_evi_true_range}
		The probability that the true MDP $M$ does not belong to the set of plausible MDPs $\mathcal{M}_k$ defined by \eqref{eq:extend_p} and \eqref{eq:extend_r} for any $k \geq 1$, is at most $\frac{\delta}{3}$, that is,
		\begin{equation}
			\mathbb{P}(\exists k \geq 1, s.t. M \notin \mathcal{M}_k) \leq \frac{\delta}{3}
		\end{equation}
	\end{theorem}

	\subsection{Stochastic Shortest Path of the Extended MDP}
	As shown in Prop. 8.5.8 of \cite{puterman_markov_1994}, the aperiodic transformation does not affect the gain of any stationary policy. In other words, for any $\pi \in \Pi^{\text{SR}}$, $g^{\alpha,\pi} = g^{\pi}$. We next explain the impact of aperiodic transformation on stochastic shortest path.
	\begin{proposition}[Theorem 2.1 of \cite{fruit_exploration-exploitation_2019}]
		\label{prop:stoc_short_path_aper_trans}
		Let time-homogeneous MDP $M$ satisfy the assumptions of Prop. \ref{prop:bellman_stoc_shortest}. Let $\alpha \in (0,1]$ and $M_{\alpha}$ be the MDP obtained after applying the aperiodic transformation of parameter $\alpha$ to $M$, which satisfies the assumptions of Prop. \ref{prop:bellman_stoc_shortest} as well. So, $h_{\mapsto s}^{\alpha *}$ is well-defined for all $a \in \mathcal{S}$. Moreover, $\alpha \cdot h_{\mapsto s}^{\alpha *} = h_{\mapsto s}^{*}$.
	\end{proposition}

	Under the event that the true MDP $M$ fall into the scope of plausible MDPs $\mathcal{M}_k$, $L^k_{\mapsto s} h_{\mapsto s}^{*} \geq L_{\mapsto s} h_{\mapsto s}^{*} = h_{\mapsto s}^{*}$, with the last equality comes from the definition in Prop. \ref{prop:bellman_stoc_shortest}. Hence, by Prop. \ref{prop:stoc_shortest_dominance}, $h_{\mapsto s}^{k*} \geq h_{\mapsto s}^{*}$. Together with the implications of Prop. \ref{prop:stoc_short_path_aper_trans}, $D_k^alpha = \frac{D}{\alpha} \leq D$. Since $g^{\alpha*}_k = g^{*} \leq r_{\max}$, together with Prop. \ref{prop:span_h_k}, $sp(h_k^{\alpha}) \leq \frac{D r_{\max}}{\alpha}$.

	\subsection{Relation between Optimal T-step Reward $v^{*,T}(s_1)$ and Average Reward $g^{*}$}

	The following lemma manifests that the $T$-step reward in the changing MDP settings could be bounded by the optimistic average reward $g^{*}$ as follows.

	\begin{lemma}[Lemma 10 of \cite{gajane_variational_2019}]
		\label{lem:hete_tstep_average}
		Under the event where the true MDP falls into the scope of plausible MDPs, for all $k$ and all $s$,
		\begin{align}
			v^{*,T}(s) \leq Tg^{*}_k + D
		\end{align}
		where $g^{*}_k \defeq \max_{\pi,M\in \mathcal{M}_k}g^{*}_k(M)$. 
	\end{lemma}
	
	\subsection{Lemmas for Section \ref{sec:result}}
	\begin{lemma}
		\label{lem:r_t_azuma}
		With a probability at least $1 - \frac{\delta}{6}$, $\forall T \geq 1$,
		\begin{align}
			-\sum_{t=1}^{T} r_t \leq & -\sum_{t=1}^{T}\sum_{a\in A_{s_t}} \pi_{k_t}(s_t,a)r(s_t,a) + 2r_{\max} \sqrt{T\ln \big(\frac{4T}{\delta}\big)}
		\end{align}
	\end{lemma}
	\begin{proof}

		For any $t \geq 1$, the $\sigma$-algebra induced by the past history of state-action pairs and rewards up to time $t$ is denoted as $\mathcal{F} = \sigma(s_1,a_1,r_1,\cdots,s_t,a_t,r_t,s_{t+1})$ where by convention $\mathcal{F}_0 = \sigma(\emptyset)$ and $\mathcal{F}_{\infty} \defeq \cup_{t \geq 0} \mathcal{F}_t$. Trivially, for all $t \geq 0$, $\mathcal{F}_t \geq \mathcal{F}_{t+1}$ and the filtration $(F_t)_{t\geq0}$ is denoted by $\mathbb{F}$. Since $k_t$ is the integer-valued random variable indexing the current episode at time $t$, $k_t$ is $\mathcal{F}_{t-1}$-measurable i.e., the past sequence $(s_1,a_1,r_1,\cdots,s_t,a_t,r_t,s_{t+1})$ fully determines the ongoing episode at time $t$. Consequently, the stationary (randomize) policy executed at time $t$ is also $\mathcal{F}_{t-1}$-measurable.

		Let us consider a stochastic process $X_t \defeq r_t(s_t, a_t) - \sum_{a\in A_{s_t}} \pi_{k_t}(s_t,a)r(s_t,a) $. The term $\sum_{a\in A_{s_t}} \pi_{k_t}(s_t,a)r(s_t,a) $ is $\mathcal{F}_{t-1}$-measurable and $\mathbb{E}[r_t(s_t, a_t) | \mathcal{F}_{t-1}] = \sum_{a\in A_{s_t}} \pi_{k_t}(s_t,a)r(s_t,a)$. Since $|X_t| \leq r_{\max}$, $(X_t,\mathcal{F}_t)_{t\geq1}$ is an MDS and we can apply Azuma's inequality in Prop. \ref{prop:azuma}, namely
		\begin{align}
			& \mathbb{P} \bigg( \sum_{t=1}^{T} \big(r_t  - \sum_{a\in A_{s_t}} \pi_{k_t}(s_t,a)r(s_t,a) \big) \leq  - 2r_{\max} \sqrt{T\ln \big(\frac{4T}{\delta}\big)} \bigg) \leq (\frac{\delta}{4T})^2 \leq \frac{\delta}{16T^2}
		\end{align}
		
		Recalling that $\sum_{n=1}^{\infty} \frac{1}{n^2} = \frac{\pi^2}{6}$ and taking a union bound for all $T \geq 1$, we have the probability at least $1 - \sum_{T=1}^{\infty} \frac{\delta}{16T^2} = 1 - \frac{\pi^2 \delta}{96} \geq 1 - \frac{\pi}{6}$ and conclude the proof. Notably, the MDS-based proof is different from the proof in \cite{jaksch_near-optimal_2010}, where the authors claim (without proof) that given state-action counts $N(s,a)$ after $T$ steps, the $r_t$ are independent random variables and apply the Hoeffding inequality in Prof. \ref{prop:hoeffding}.
	\end{proof}

	\begin{lemma}
		\label{lem:delta_k_decompose_p_r}
		Under the event that the true MDP falls into the scope of plausible MDPs ($M \in \mathcal{M}_k, \forall k$), $\Delta_k$ could be upper bounded by 
		\begin{equation}
			\Delta_k \leq \Delta_k^{p} + \Delta_k^{r} + \frac{3\epsilon_k}{2} \sum_{s\mathcal{S}}\nu_k(s)
		\end{equation}
		where 
		\begin{displaymath}
			\Delta_k^{p} \defeq \alpha \sum\limits_{s \in \mathcal{S}} \nu_{k}(s)\bigg(\sum\limits_{\substack{a \in \mathcal{A}_{s}\\ s^{\prime} \in \mathcal{S}}} \pi_{k}(a | s) p_{k}\big(s^{\prime} | s, a\big) h_{k}\big(s^{\prime}\big)-h_{k}(s)\bigg)
		\end{displaymath}
		and 
		\begin{displaymath}
			\Delta_{k}^{r} \defeq \sum\limits_{s \in \mathcal{S}} \sum\limits_{a \in \mathcal{A}_{s}} \nu_{k}(s) \pi_{k}(a | s)\big(r_{k}(s, a)-r(s, a)\big).
		\end{displaymath}
	\end{lemma}
	
	Recalling that $\Delta_k^{p1} \defeq \alpha \sum_{s,a,s'} \nu_{k}(s) \pi_{k}(a | s) \big(p_{k}(s^{\prime} | s, a) - p(s^{\prime} | s, a) \big) h_{k}(s^{\prime})$. If we define $\Delta_k^{p3} \defeq \alpha \sum\limits_{s,a,s'} \nu_{k}(s,a) \bigg(p_{k}\big(s^{\prime} | s, a\big) - p\big(s^{\prime} | s, a\big) \bigg) h_{k}\big(s^{\prime}\big)$, $p_k(s^{\prime}| s) \defeq \sum_a \pi_k (a| s) p_k(s'|s,a)$, $\bar{p}_k(s^{\prime}| s) \defeq \sum_a \pi_k (a| s) p(s'|s,a)$, and $\bar{p}_k(s^{\prime}| s) \defeq \sum_a \pi_k (a| s) p(s'|s,a)$, we can have the following lemma.
	\begin{lemma}
		\label{lem:delta_p1_p3}
		Under the case where the true MDP falls into the scope of plausible MDPs ($M \in \mathcal{M}_k, \forall k$), with probability at least $1-\frac{\delta}{6}$, we have $\sum\limits_{k=1}^{k_T} \Delta_k^{p1} \leq \sum\limits_{k=1}^{k_T} \Delta_k^{p3} +  4 D r_{\max} \sqrt{T \ln\big(\frac{6T}{\delta}\big)}$.
	\end{lemma}
	\begin{proof}
		The proof is similar to that of Lemma \ref{lem:r_t_azuma}.
		
		Let us consider a stochastic process $X_t$ defined as
		\begin{align}
			X_t &\defeq \alpha \sum_{a,s^{\prime}} \pi_{k_t}(a|s_t) p_{k_t} (s^{\prime}|s_t,a) h_{k_t}(s^{\prime})  -\alpha \sum_{s^{\prime}} p_{k_t} (s^{\prime}|s_t,a_t) h_{k_t}(s^{\prime})\\
			& \stackrel{(a)}{=} \alpha \sum_{a,s^{\prime}} \pi_{k_t}(a|s_t) p_{k_t} (s^{\prime}|s_t,a) w_t(s^{\prime})  -\alpha \sum_{s^{\prime}} p_{k_t} (s^{\prime}|s_t,a_t) w_t (s^{\prime})\nonumber
		\end{align}
		where $w_t \defeq h_{k_t} + \lambda_t e$ with $\lambda_t$ being any constant. The equation $(a)$ comes from the fact that for a given $s_t$, $\sum_{a,s^{\prime}} \pi_{k_t}(a|s_t) p_{k_t} (s^{\prime}|s_t,a) =1$ and $\sum_{s^{\prime}} p_{k_t} (s^{\prime}|s_t,a_t) =1$. If we take $\lambda_t \defeq -\frac{1}{2} (\min h_{k_t} + \max h_{k_t})$, by Prop. \ref{prop:span_h_k} and Prop. \ref{prop:stoc_short_path_aper_trans}, $\Vert w_t (s^{\prime}) \Vert_{\infty} \leq \frac{sp(h_{k_t}^{\alpha})}{2} = \frac{D r_{\max}}{2\alpha} $, we have $\vert X_t \vert \leq 2\alpha \Vert w_t (s^{\prime}) \Vert_{\infty} \leq D r_{\max}$ almost sure for all $t$. Notably, 
		\begin{align}
			\sum_{t=1}^{T} X_t 
			= \alpha \sum\limits_{k=1}^{k_T}\sum\limits_{s,a,s'} \left( \nu_{k}(s) \pi_{k}(a | s) - \nu_{k}(s,a) \right) p_{k}(s^{\prime} | s, a) h_{k}(s^{\prime})  
		\end{align}
		
		Similar to Lemma \ref{lem:r_t_azuma}, by Azuma's inequality, 
		\begin{align}
			\mathbb{P} \left( \sum_{t=1}^T X_t \geq 2 D r_{\max} \sqrt{T \ln\big(\frac{6T}{\delta}\big)} \right) \leq \frac{\delta}{36T^2}
		\end{align}

		Similarly, replacing $p_k$ in $X_t$ by $p$, we have a similar bound. 

		Therefore, for a fixed $T$, with a probability at least $\frac{\delta}{36T^2}$, \begin{align}
			\sum\limits_{k=1}^{k_T} \Delta_k^{p1} \leq \sum\limits_{k=1}^{k_T} \Delta_k^{p3} +  4 D r_{\max} \sqrt{T \ln\big(\frac{6T}{\delta}\big)}
		\end{align}
		
		Taking a union bound on $T$, we have with a probability at least $1 - \sum_{T=1}^{\infty}\frac{\delta}{36T^2} \geq 1 - \frac{\delta}{6}$, we have \begin{align}
			\sum\limits_{k=1}^{k_T} \Delta_k^{p1} \leq \sum\limits_{k=1}^{k_T} \Delta_k^{p3} +  4 D r_{\max} \sqrt{T \ln\big(\frac{6T}{\delta}\big)}
		\end{align}
	\end{proof}
	\begin{lemma}
		\label{lem:delta_p3}
		Under the case where the true MDP falls into the scope of plausible MDPs ($M \in \mathcal{M}_k, \forall k$), if $\hat{V}_p$ is set as the true value $V_p^T$, $\Delta_k^{p3} \leq Dr_{\max} \sum\limits_{s,a} \nu_{k}(s,a) \beta_{p,k}^{sa}$, where $\beta_{p,k}^{sa} \defeq \sum_{s^{\prime}}\beta_{p,k}^{sas^{\prime}}$.
	\end{lemma}
	\begin{proof}
		We bound $\Delta_k^{p3}$ as
		\begin{align}
			\Delta_k^{p3} & = \alpha \sum\limits_{s,a,s'} \nu_{k}(s,a) \bigg(p_{k}\big(s^{\prime} | s, a\big) - p\big(s^{\prime} | s, a\big) \bigg) h_{k}\big(s^{\prime}\big)\nonumber\\
			& \stackrel{(a)}{=} \alpha \sum\limits_{s,a,s'} \nu_{k}(s,a) \big(p_{k}(s^{\prime} | s, a) - p(s^{\prime} | s, a) \big) w_{k}^{s}\big(s^{\prime}\big)
			\nonumber\\
			& \stackrel{(b)}{\leq} \alpha \sum\limits_{s,a} \nu_{k}(s,a) \vert p_{k}(\cdot | s, a) - p(\cdot | s, a) \vert_1 \Vert w_{k}^{s}(\cdot)\Vert_{\infty} 
			\nonumber\\
			& \stackrel{(c)}{\leq} \frac{Dr_{\max}}{2} \sum\limits_{s,a} \nu_{k}(s,a) \vert p_{k}(\cdot | s, a) - p(\cdot | s, a) \vert_1 
			\nonumber
			\\
			& \stackrel{(d)}{\leq} Dr_{\max} \sum\limits_{s,a} \nu_{k}(s,a) \beta_{p,k}^{sa}
			\nonumber
		\end{align} 
		where the equation $(a)$ comes by applying a constant shift same as in Proof of Lemma \ref{lem:delta_p1_p3} with 
		$\lambda_k \defeq -\frac{1}{2} (\min h_{k_t} + \max h_{k_t})$ and $w_k \defeq h_{k} + \lambda_k e$. The inequality $(b)$ comes from the H\"older inequality and the inequality $(c)$ is due to $\Vert w_k \Vert_{\infty} \leq \frac{Dr_{\max}}{2}$. Based on the triangle inequality, we have the inequality $(d)$ as
		\begin{align}
			& \Vert p_{k}(\cdot | s, a) - p(\cdot | s, a) \Vert_1 \nonumber \\
			\stackrel{(e)}{\leq} &  \Vert p_{k}(\cdot | s, a) - \hat{p}_k(\cdot | s, a) \Vert_1 + \Vert p(\cdot | s, a) - \hat{p}_k(\cdot | s, a) \vert_1\nonumber \\
			\stackrel{(f)}{\leq} & 2\beta_{p,k}^{sa} + \hat{V}_p + V_p^T
			= 2\beta_{p,k}^{sa} + 2V_p^T
		\end{align}
		where the inequality $(e)$ comes from the triangle inequality. Meanwhile, by construction $p_{k}(\cdot | s, a) \in \mathcal{B}_p^{k}$, for any $s^{\prime} \in \mathcal{S}$, $\vert p_{k}(s^{\prime} | s, a) - \hat{p}_k(s^{\prime} | s, a) \vert < \beta_{p,k}^{sas^{\prime}}$. Hence, $\Vert p_{k}(\cdot | s, a) - \hat{p}_k(\cdot | s, a) \Vert_1 < \beta_{p,k}^{sa} + \hat{V}_p$. On the other hand, since the true MDP falls into the scope of plausible MDPs ($M \in \mathcal{M}_k, \forall k$), $p(\cdot | s, a) \in \mathcal{B}_p^{k}$, we have $\Vert p(\cdot | s, a) - \hat{p}_k(\cdot | s, a) \Vert_1 < \beta_{p,k}^{sa} + V_p$. Therefore, if $\hat{V}_p$ is set as the true value $V_p^T$, we obtain the inequality $(f)$.

		The conclusion comes.
	\end{proof}
	\begin{lemma}
		\label{lem:delta_p2}
		Under the case where the true MDP falls into the scope of plausible MDPs ($M \in \mathcal{M}_k, \forall k$), with probability at least $1-\frac{\delta}{6}$,  we have 
		\begin{align}
			\sum_{k=1}^{k_{T}} \Delta_{k}^{p2} & \leq 2 D r_{\max} \sqrt{T \ln\big(\frac{6T}{\delta}\big)} + k_T D r_{\max} \nonumber
		\end{align}
	\end{lemma}
	\begin{proof}
		\begin{align}
			& \sum_{k=1}^{k_{T}} \Delta_{k}^{p2} \\
			=& \alpha \sum_{k=1}^{k_{T}} \sum\limits_{s} \nu_{k}(s)\bigg(\sum\limits_{\substack{a,s^{\prime}}} \pi_{k}(a | s) p\big(s^{\prime} | s, a\big) h_{k}\big(s^{\prime}\big)-h_{k}(s)\bigg)\nonumber\\
			\stackrel{(a)}{=} & \alpha \sum_{k=1}^{k_{T}} \sum\limits_{s} \nu_{k}(s)\bigg(\sum\limits_{\substack{a,s^{\prime}}} \pi_{k}(a | s) p\big(s^{\prime} | s, a\big) w_{k}\big(s^{\prime}\big)-w_{k}(s)\bigg)\nonumber\\
			= & \alpha \sum_{k=1}^{k_{T}} \sum_{t = t_k}^{t_{k+1}-1} \bigg(\sum\limits_{\substack{a, s^{\prime}}} \pi_{k}(a | s) p\big(s^{\prime} | s, a\big) w_{k}\big(s^{\prime}\big)-w_{k}(s_{t}) \bigg) \nonumber\\
			= & \alpha \sum_{k=1}^{k_{T}} \sum_{t = t_k}^{t_{k+1}-1} \bigg(\sum\limits_{\substack{a, s^{\prime}}} \pi_{k}(a | s) p\big(s^{\prime} | s, a\big) w_{k}\big(s^{\prime}\big)-w_{k}(s_{t+1}) + w_{k}(s_{t+1}) - w_{k}(s_{t}) \bigg) \nonumber\\
			= & \alpha \sum_{k=1}^{k_{T}} \sum_{t = t_k}^{t_{k+1}-1} \bigg(\sum\limits_{\substack{a, s^{\prime}}} \pi_{k}(a | s) p\big(s^{\prime} | s, a\big) w_{k}\big(s^{\prime}\big)-w_{k}(s_{t+1}) \bigg)  + \alpha \sum_{k=1}^{k_{T}} \big( w_{k}(s_{t_{k+1}}) - w_{k}(s_{t}) \big) \nonumber
		\end{align}
		where the equation $(a)$ comes by applying a constant shift same as in Proof of Lemma \ref{lem:delta_p1_p3} with 
		$\lambda_k \defeq -\frac{1}{2} (\min h_{k_t} + \max h_{k_t})$ and $w_k \defeq h_{k} + \lambda_k e$.

		Applying similar methodology in Lemma \ref{lem:r_t_azuma} and \ref{lem:delta_p1_p3}, we have 
		\begin{displaymath}
			\alpha \sum_{k=1}^{k_{T}} \sum_{t = t_k}^{t_{k+1}-1} \bigg(\sum\limits_{\substack{a, s^{\prime}}} \pi_{k}(a | s) p\big(s^{\prime} | s, a\big) w_{k}\big(s^{\prime}\big)-w_{k}(s_{t+1}) \bigg) \leq 2 D r_{\max} \sqrt{T \ln\big(\frac{6T}{\delta}\big)}
		\end{displaymath}
		with a probability at least $1-\frac{\delta}{6}$. On the other hand, $w_{k}(s_{t_{k+1}}) - w_{k}(s_{t}) \leq \frac{Dr_{\max}}{\alpha}$.

		The conclusion comes.
	\end{proof}

	\begin{lemma}
		\label{lem:delta_r}
		Under the case where the true MDP falls into the scope of plausible MDPs ($M \in \mathcal{M}_k, \forall k$), if $\hat{V}_r$ is set as the true value $V_r^T$, with probability at least $1-\frac{\delta}{6}$, we have $\sum\limits_{k=1}^{k_T} \Delta_k^{r} \leq 4 r_{\max} \sqrt{T \ln\big(\frac{4T}{\delta}\big)} + 2 \sum\limits_{k=1}^{k_T}\sum\limits_{s,a} \nu_{k}(s,a) (\beta_{r,k}^{sa}+V_r^T)$.
	\end{lemma}
	\begin{proof}
		The proof is similar to that of Lemma \ref{lem:delta_p1_p3}.
		
		Denote $\Delta_k^{r1} \defeq \sum\limits_{s \in \mathcal{S}, a \in \mathcal{A}_{s}} \nu_{k}(s,a) \big(r_{k}(s, a)-r(s, a)\big)$. By Azuma's inequality in Prop. \ref{prop:azuma}, $\sum\limits_{k=1}^{k_T} \Delta_k^{r}  \leq \sum\limits_{k=1}^{k_T} \Delta_k^{r1} + 4 r_{\max} \sqrt{T \ln\big(\frac{4T}{\delta}\big)}$. Meanwhile, analogously to Lemma \ref{lem:delta_p3}, $\Delta_k^{r1} \leq 2\sum\limits_{s,a} \nu_{k}(s,a) (\beta_{r,k}^{sa}+V_r^T)$.
	\end{proof}

	% \begin{proposition}[Theorem 1 of \cite{gajane_variational_2019}]
	% 	If a global notion of variation in average reward is defined as 
	% 	\begin{align}
	% 		V^T \defeq \sum_{t=1}^{T-1} \vert g^{*}(M_{t+1}) - g^{*}(M_{t})\vert
	% 	\end{align}
	% 	we have 
	% 	\begin{align}
	% 		V^T \leq V_r^T + D r_{\max} V_p^T
	% 	\end{align}
	% \end{proposition}
	% \begin{proof}
	% 	The proof could comes based on Lemma \ref{lem:two_mdp_diff}. 
	% \end{proof}

	\begin{proposition}
		\label{prop:sum_a_inequaility}
		For any proof $n \geq 1$ and any $n$-tuple $(a_1,\cdots,a_n) \in \mathbb{R}^{n}$, $(\sum_i^n a_i)^2 \leq n\sum_i^n (a_i)^2$. 
	\end{proposition}
	\begin{proof}
		The statement is trivially true for $n=1$. For $n=2$, we have $(a_1 + a_2)^2 = a_1^2 + a_2^2 + 2a_1a_2 \leq 2(a_1^2 + a_2^2)$. Next we prove results for $n \geq 2$ by induction. Assume that for $n \geq 2$, $(\sum_{i=1}^n a_i)^2 \leq n\sum_{i=1}^n (a_i)^2$. Then for $n+1$, we have
		\begin{align}
			(\sum_{i=1}^{n+1} a_i)^2 & = (\sum_{i=1}^{n} a_i + a_{n+1})^2 \nonumber\\
			& = (\sum_{i=1}^{n} a_i)^2  + a_{n+1}^2 + 2 a_{n+1}\sum_i^{n} a_i\nonumber\\
			& \leq n\sum_{i=1}^n (a_i)^2 + a_{n+1}^2 + 2 a_{n+1}\sum_{i=1}^{n} a_i\\
			& \leq n\sum_{i=1}^n (a_i)^2 + a_{n+1}^2 + \sum_{i=1}^{n} (a_{n+1}^2 + a_i^2) \nonumber\\
			& = (n+1) \sum_{i=1}^{n+1} (a_i)^2\nonumber
		\end{align}
		We conclude the proof.
	\end{proof}
	\begin{proposition}
		\label{prop:sum_to_ln}
		For any sequence of numbers $z_1, \cdots, z_i, \cdots$ with $0 \leq z_i \leq Z_{i-1} \defeq \max \big\{1,\sum_{i=1}^{k-1} z_i\big\}$, we have for $n \geq 1$,
		\begin{align}
			\sum_{i=1}^{n} \frac{z_i}{Z_i} \leq 2 + 2\ln(Z_{n+1})
		\end{align}
		and
		\begin{align}
			\sum_{i=1}^{n} \frac{z_i}{\sqrt{Z_i}} \leq 3\sqrt{Z_{i+1}}
		\end{align}
	\end{proposition}
	\begin{proof}
		We first show that the proposition holds for all $n$ with $Z_{n} \leq 1$. In this case, $\sum_{i=1}^{n + 1} \frac{z_i}{Z_i} = \sum_{i=1}^{n} \frac{z_i}{Z_i} + \frac{z_{n+1}}{Z_{n+1}} \leq 1 + 1 < 2 + 2 \ln(Z_{n+1})$. Notably, this implies the proposition holds for $n = 1$, since $Z_n = \max\{1,\sum_{i=1}^{n-1} z_i\} =\max\{1,0\}=1$.

		Next, we show that the proposition holds by induction. Assume that $\sum_{i=1}^{n + 1} \frac{z_i}{Z_i} = \sum_{i=1}^{n} \frac{z_i}{Z_i} + \frac{z_{n+1}}{Z_{n+1}} \leq 1 + 1 < 2 + 2 \ln(Z_{n+1})$ holds for $n = k \geq 1$ with $Z_{n} \geq 1$. For $n = k+1$, we have 
		\begin{align}
			\sum_{i=1}^{k + 1} \frac{z_i}{Z_i} & = \sum_{i=1}^{k} \frac{z_i}{Z_i} + \frac{z_{k+1}}{Z_{k+1}} \nonumber\\
			& \leq 2 + 2 \ln(Z_{k+1}) +  \frac{z_{k+1}}{Z_{k+1}}\nonumber\\
			& = \leq 2 + 2 \ln(Z_{k+1}) +  \frac{Z_{k+2} - Z_{k+1}}{Z_{k+1}}\\
			& \stackrel{(a)}{\leq} 2 + 2 \ln(Z_{k+2})\nonumber
		\end{align}
		where the equation $(a)$ holds since $\frac{Z_{k+2}}{Z_{k+1}} \in [1,2]$ and $2\ln(t) \geq t - 1$ for $t \in [1,2]$.

		Similarly, we can have the proof of the latter part.
	\end{proof}
	\begin{proposition}
		\label{prop:bound_p3_prop1}
		It holds almost surely that for all $k \geq 1$ and for all $(s,a,s') \in \mathcal{S}\times \mathcal{A} \times \mathcal{S}$:
		\begin{align}
			& \sum_{s'\in \mathcal{S}} \sqrt{\hat{p}_k(s'|s,a)(1- \hat{p}_k(s'|s,a))} 
			\leq  \sqrt{\Gamma(s,a)-1} 
		\end{align}
	\end{proposition}
	\begin{proof}
		Define $\mathcal{S}_k^{(s,a)} = \{s^{\prime} \in \mathcal{S}: \hat{p}_k(s^{\prime}|s,a) > 0\}$. Using Cauchy-Schwarz inequality, we have 
		\begin{align}
			& \sum_{s'\in \mathcal{S}} \sqrt{\hat{p}_k(s'|s,a)(1- \hat{p}_k(s'|s,a))} \nonumber\\
			= & \sum_{s'\in \mathcal{S}_k^{(s,a)}} \sqrt{(1- \hat{p}_k(s'|s,a)) \hat{p}_k(s'|s,a)} \nonumber\\
			\leq &  \sqrt{\sum_{s'\in \mathcal{S}_k^{(s,a)}}(1- \hat{p}_k(s'|s,a)) \sum_{s'\in \mathcal{S}_k^{(s,a)}}\hat{p}_k(s'|s,a)} \nonumber\\
			\leq &  \sqrt{\sum_{s'\in \mathcal{S}}(1- \hat{p}_k(s'|s,a)) \sum_{s'\in \mathcal{S}_k^{(s,a)}}\hat{p}_k(s'|s,a)} \nonumber\\
			\leq & \sqrt{\Gamma(s,a)-1} \nonumber
		\end{align}
		
		The lemma comes.
	\end{proof}
	\begin{proposition}[Prop. 18 of \cite{jaksch_near-optimal_2010}]
		\label{prop:bound_k_t}
		For all $T \geq SA$, the number of episodes $k_T$ up to step $T \geq SA$ is upper bounded as $SA \log_2 \big( \frac{8T}{SA}\big)$.
	\end{proposition}
	
	\begin{proposition}
	\label{prop:bound_ucrlb_comb_s}		
	\eqref{eq:bound_ucrlb_comb_f} could be further simplified as 
	\begin{align}
		& \Delta(\text{VB-UCRL}, T)  \nonumber\\
		\leq & \max(r_{\max},Dr_{\max}) \bigg(43 \sqrt{T\ln\big(\frac{T}{\delta}\big) \sum_{s,a}\Gamma(s,a)}  \\
		& \quad + 72 S^2A \ln\big(\frac{T}{\delta}\big) \ln(T) \bigg) +  Dr_{\max}T V_p^T  +  2TV_r^T \nonumber
	\end{align}
	\end{proposition}
	\begin{proof}
		For $T < 6SA$, we can directly have 
		\begin{align}
			\Delta(\text{VB-UCRL}, T) & \leq r_{\max}T = r_{\max} \sqrt{T} \sqrt{T} \\
			& \leq r_{\max} \sqrt{6SAT} \leq \sqrt{6T \sum_{s,a}\Gamma(s,a)} \nonumber
		\end{align}
		Also, if $1 \leq T \leq 43^2 A \log \big(\frac{T}{\delta} \big)$, we have $T^2 \leq 43^2 AT \log \big(\frac{T}{\delta}\big)$ (or $T\leq \sqrt{AT \log \big(\frac{T}{\delta}\big)}$), thus  $\Delta(\text{VB-UCRL}, T) \leq r_{\max}\sqrt{AT \log \big(\frac{T}{\delta}\big)}$.

		For $T \geq 6SA$, we have $6SAT \leq T^2$, thus $\ln \big(\frac{6SAT}{\delta}\big) \leq \ln \big(\frac{T^2}{\delta}\big) \leq 2\ln \big(\frac{T}{\delta}\big)$. Similarly, if $T \geq 43^2 A \log \big(\frac{T}{\delta} \big)$, we have $A \leq \frac{\sqrt{A T \log \big(\frac{T}{\delta} \big)}}{43 \log \big(\frac{T}{\delta} \big)}$. Together with $\log(8T) \leq 2\log(T)$, $\log (\frac{8T}{SA}) \leq \frac{2}{43} \sqrt{A T \log \big(\frac{T}{\delta} \big)} $ 

		Thus, we can simplify \eqref{eq:bound_ucrlb_comb_f} as 
		\begin{align}
			& \Delta(\text{VB-UCRL}, T)  \nonumber\\
			\leq & \max(r_{\max},Dr_{\max}) \bigg(\sqrt{T\ln\big(\frac{T}{\delta}\big) \sum_{s,a}\Gamma(s,a)} \times  (6\sqrt{2}+ 6\sqrt{2} + 12\sqrt{2} + \frac{2}{43}) \\
			& \quad + S^2A \ln\big(\frac{T}{\delta}\big) \ln(T) (24 + 48) \bigg) +  Dr_{\max}T V_p^T  +  2TV_r^T \nonumber\\
			\leq & \max(r_{\max},Dr_{\max}) \bigg(43 \sqrt{T\ln\big(\frac{T}{\delta}\big) \sum_{s,a}\Gamma(s,a)}  \nonumber \\
			& \quad + 72 S^2A \ln\big(\frac{T}{\delta}\big) \ln(T) \bigg) +  Dr_{\max}T V_p^T  +  2TV_r^T \nonumber
		\end{align}

		Replacing $\delta^{\prime} = \frac{5}{6}\delta$, and taking account of $\log \big(\frac{T}{\delta^{\prime}} \big) = \log \big(\frac{6T}{5\delta}\big) < 2 \log \big(\frac{T}{\delta}\big)$, we have the conclusion.
	\end{proof}
	\bibliographystyle{IEEEtran}
	\bibliography{bib}
\end{document}